\documentclass{article}
\usepackage[main, preprint]{neurips_2026}

\usepackage[utf8]{inputenc}
\usepackage[T1]{fontenc}
\usepackage{hyperref}
\usepackage{url}
\usepackage{booktabs}
\usepackage{amsfonts}
\usepackage{nicefrac}
\usepackage{microtype}
\usepackage{xcolor}
\usepackage{amsmath}
\usepackage{amssymb}
\usepackage{amsthm}
\usepackage{algorithm}
\usepackage{algorithmic}
\usepackage{graphicx}
\usepackage{subcaption}
\usepackage{multirow}

\newtheorem{assumption}{Assumption}[section]
\newtheorem{proposition}[assumption]{Proposition}
\newtheorem{theorem}[assumption]{Theorem}
\newtheorem{corollary}[assumption]{Corollary}
\newtheorem{lemma}[assumption]{Lemma}

\theoremstyle{remark}

\makeatletter
\newcommand{\blfootnote}[1]{%
  \begingroup
  \renewcommand\thefootnote{}\footnotetext{#1}%
  \addtocounter{footnote}{-1}%
  \endgroup
}
\makeatother

\title{Learning Deployable Causal Action Geometry under Temporal Non-Stationarity}

\author{
\textbf{Changjian Liu}$^{1,2}$ \thanks{Equal contribution.} \    \thanks{Work done during an internship at Taobao, Alibaba Group.}
\quad
\textbf{Tianyu Wang}$^{3}$\footnotemark[1]
\quad
\textbf{Yuwei Xu}$^{4}$\footnotemark[1] \ \footnotemark[2]
\quad
\textbf{Xiaoxuan Deng}$^{3}$
\quad
\textbf{Zhilin Zhang}$^{5}$
\\
\textbf{Yong Gao}$^{1,2}$\thanks{Corresponding authors: Yong Gao (\texttt{gaoyong@pku.edu.cn}) and Jian Xu (\texttt{xiyu.xj@taobao.com}).}
\quad
\textbf{Chuan Yu}$^{3}$
\quad
\textbf{Jian Xu}$^{3}$\footnotemark[3]
\quad
\textbf{Bo Zheng}$^{3}$
\\[0.5em]
$^{1}$ Institute of Remote Sensing and Geographic Information System, \\ School of Earth and Space Sciences, Peking University, Beijing, China \\
$^{2}$ Beijing Key Laboratory of Spatio-temporal Perception and Urban Resilience \\
$^{3}$ Taobao \& Tmall Group of Alibaba, Beijing, China \\ 
$^{4}$ The Chinese University of Hong Kong, Shenzhen, China 
$^{5}$ Independent Researcher 
}

\begin{document}
\maketitle
\blfootnote{Accepted at the 40th Conference on Neural Information Processing Systems (NeurIPS 2026).}
\vspace{-0.5em}

\begin{abstract}
Learning a response law for continuous decisions from observational time series is central to many real-world decision systems, yet temporal non-stationarity poses a fundamental challenge: as outcomes drift over time, naively pooled models can conflate action effects with shifting baselines and even reverse action rankings. Even when the one-step response is identifiable, unrestricted cross-time generalization remains fundamentally ill-posed. We therefore seek the part of the response that remains stable for decisions even when outcome levels do not. We propose deployable causal action geometry: the action-response shape that remains after anchoring away action-independent level shifts on a selected link scale and profiling out positive amplitude variation. Building on this target, we develop a two-stage orthogonal estimator that first obtains cross-fitted pilot estimates of anchored responses and then performs honest profiled selection over link, representation, and shape classes. The resulting estimator admits oracle guarantees for the restricted target and, under approximate persistence, produces decision rules that remain near optimal in future periods. Experiments on synthetic data and Dominick's support the proposed target and its transfer behavior. In a randomized Taobao A/B test, the deployed system improves GMV and purchase counts by \(1.16\%\) and \(1.70\%\).
\end{abstract}

\section{Introduction}
\label{sec:introduction}
Continuous decision making from observational logs is central to many large-scale platforms. A prominent example motivating this work is the budget decision in online shopping platforms. The system repeatedly adjusts spending under evolving traffic, policy, and demand conditions. Similar challenges arise in advertising, pricing, and other resource allocation systems, where actions are continuous and decisions must be made from logged time-series data rather than repeated randomized experimentation \citep{hua2021markdowns, multichanneluplift}. In such settings, the goal is not merely to predict realized outcomes, but to learn how outcomes would change under counterfactual action levels, so that future actions can be optimized \citep{athey2021policy}.

This setting raises two distinct challenges that are often conflated. The first is confounding: logged actions are chosen adaptively from histories that also affect outcomes, a problem extensively studied in continuous treatment effect estimation \citep{kennedy2017continuous,nie2021vcnet,hirano2005propensity} and counterfactual sequence modeling \citep{lim2018rmsn,bica2020crn,frauen2025temporalmetalearners,melnychuk2022causaltransformer}. The second, and the focus of this paper, is temporal non-stationarity. Demand, traffic, policies, and market conditions change over time, causing outcome levels to drift even when the causal effect of an action remains structurally similar. As a result, a pooled predictive model can fit historical outcomes well while absorbing time-varying baselines into the learned action surface. This creates a target mismatch: predictive accuracy on logged data need not imply correct action rankings for future deployment \citep{crasson2024causalforecasting,schultz2023causalpricing}. Existing invariant and robust prediction frameworks \citep{peters2016icp,arjovsky2019irm,rothenhausler2021anchor,krueger2021out} provide important tools for cross-environment stability, but they do not directly define a causal response target for confounded continuous actions under temporal drift.

This distinction leads to two separate questions. The first is identification within an observed period: under standard one-step exchangeability and overlap conditions, the period-specific response is identifiable. The second is transfer across periods: which part of these responses can be used to rank actions in future environments? We show that pooling observations across time can mix action-sensitive variation with drifting baseline structure and can even reverse the action ranking relevant for deployment. More fundamentally, past periods alone do not identify an unrestricted future response law. Thus the main obstacle is not only estimation under distribution shift, but the absence of a well-defined deployable target.

We address this by introducing a restricted target, which we call deployable causal action geometry. Rather than fitting raw outcome levels, we compare actions through anchored contrasts on a family of monotone link scales. Anchoring removes baseline shifts that do not depend on the action. The link scale and a profiled calibration factor isolate the shared action shape that determines intervention rankings. Identity and log links play a canonical role: they remove additive and multiplicative calibration. The resulting target allows period-specific levels and amplitudes to vary while requiring only the action-response shape relevant for decision making to persist. We estimate this target with an honest two-stage procedure. The estimator follows the orthogonal or debiased learning principle \citep{chernozhukov2018double,foster2019orthogonal}. The first stage constructs cross-fitted orthogonal pilots for anchored responses under standard confounding assumptions. The second stage performs profiled selection over the link function, representation, and shape class, yielding a decision rule over continuous actions. We prove an oracle inequality over a finite candidate library and show that, under approximate persistence of the learned geometry, the induced rule has controlled regret in future periods.

Our contributions are threefold. First, we formalize the target mismatch induced by pooled regression under temporal drift and show that unrestricted future response learning is generally ill posed. SSecond, we define a deployable causal action geometry based on anchored link-scale contrasts and develop an honest orthogonal estimator with an oracle inequality over a finite candidate library. Third, we characterize deployability under approximate geometry persistence, separating the structural cost of persistence violations from the statistical error of learning the geometry. We validate these claims in synthetic experiments that isolate link selection, nuisance robustness, overlap stress, and persistence violations. We further evaluate the method on the Dominick's dataset and in a large-scale Taobao pricing deployment, including randomized online A/B testing.

\section{Problem Setup}
\label{sec:problem_setup}
\subsection{One-Step Response under Temporal Drift}
\paragraph{Observed longitudinal data.}
We consider continuous one-step decision making from temporally non-stationary observational logs. Suppressing the unit index, let \(H_t\) denote the pre-treatment history available before period \(t\), \(A_t\in\mathcal A\subseteq \mathbb R\) the continuous action, and \(Y_{t+1}\) the next-step outcome. In the pricing application, \(H_t\) contains item attributes and recent sales and price histories, \(A_t\) is the normalized price intervention, and \(Y_{t+1}\) is next-step demand. 

\paragraph{Temporal non-stationarity and estimand scope}
Temporal non-stationarity means that the joint law of \((H_t,A_t,Y_{t+1})\) changes with \(t\). We decompose this variation into three sources of drift: \(p_t(h,a,y)=p_t(y\mid a,h)\,p_t(a\mid h)\,p_t(h),\) corresponding to changes in the outcome mechanism \(p_t(y\mid a,h)\), the logging policy \(p_t(a\mid h)\), and the history distribution \(p_t(h)\). The target mismatch below is driven by the interaction of these components: different action values may be observed under different mixtures of temporal environments, so \(Y\mid A=a\) does not necessarily represent a pure intervention response.
Our target is a proximal one-step response, not a long-horizon dynamic treatment effect. The question is therefore not only within-period identification, but whether the induced action ranking remains meaningful under temporal drift.

\paragraph{Within-Time Causal Description}
For each \(a\in\mathcal A\), let \(Y_{t+1}(a)\) denote the one-step potential outcome, while leaving the realized pre-treatment history \(H_t\). The period-specific causal response is $m_t(h,a) := \mathbb E_t\!\left[Y_{t+1}(a)\mid H_t=h\right]$.

\begin{assumption}[One-Step Identification]
\label{ass:one_step_identifiability}
For each $t$, there exists $\mathcal T_t\subseteq\mathcal H\times\mathcal A$ such that $\forall (h,a)\in\mathcal T_t$: (i) \textbf{Consistency:} $A_{i,t}=a \implies Y_{i,t+1}=Y_{i,t+1}(a)$ a.s. (ii) \textbf{No Interference:} $Y_{i,t+1}(a_i,\mathbf a_{-i}) = Y_{i,t+1}(a_i)$ $\forall \mathbf a_{-i}$. (iii) \textbf{Overlap:} $\exists$ open $\mathcal N(a)\ni a$ s.t.\ $f_{A_t\mid H_t}(a'\mid h)>0$ $\forall a'\in\mathcal N(a)$. (iv) \textbf{One-Step Exchangeability:} $Y_{i,t+1}(a) \perp A_{i,t} \mid H_{i,t}$.
\end{assumption}

\begin{proposition}[Within-time Identification]
\label{prop:pointwise_identification_setup}
Under Assumption~\ref{ass:one_step_identifiability}, $\forall t$ and a.e.\ $(h,a)\in\mathcal T_t$,
$ m_t(h,a) = \mathbb E_t \!\left[Y_{i,t+1}\mid H_{i,t}=h,A_{i,t}=a\right]$
\end{proposition}

The proof is standard and deferred to Appendix~\ref{app:proof_pointwise_identification}. Thus, temporal drift does not by itself preclude within-period identification. The remaining issue is deployment: a response learned from pooled historical
periods may target the wrong object.

\begin{figure}[!t]
    \centering
    \includegraphics[width=0.8\linewidth]{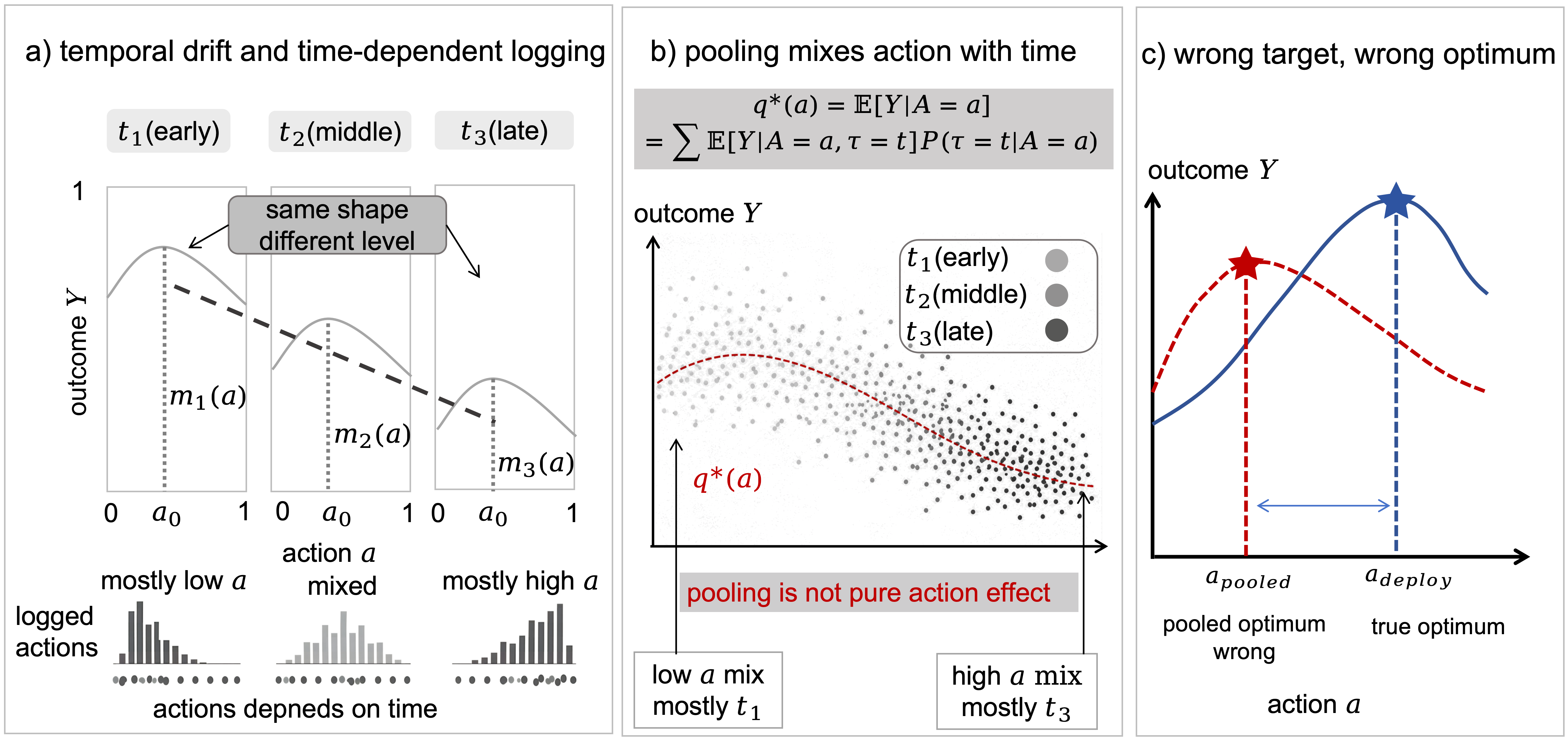}
    \caption{
\textbf{Target mismatch under drift.}
Pooled data under non-stationary environments compares confounded temporal mixtures, leading to action ranking reversal despite good historical fit.
}
    \label{fig:target_mismatch}
\end{figure}

\subsection{Target mismatch of pooled regression}
To formalize pooled learning, let \(\tau\) denote the period label of a randomly sampled training observation, with \(\mathbb P(\tau=t)=\omega_t\). Conditional on \(\tau=t\), the observation \((H,A,Y)\) follows the period-\(t\) law. We write \(\mathbb E_{\rm pool}\) for expectation under this mixture. A standard pooled predictive learner ignores \(\tau\) and targets
\(q^\star(x,a):=\mathbb E_{\rm pool}\!\left[Y\mid X=x,A=a\right],\) where \(X=\phi(H).\)
This is a valid factual prediction target under the pooled training distribution, but it need not be a deployable response surface. The mismatch arises because conditioning on \(A=a\) can change the mixture over
periods and histories; hence different action values may be evaluated under different temporal environments.

To expose the mismatch, consider an additive response decomposition \(m_t(h,a)=b_t(h)+r(h,a),\) where \(b_t(h)\) is an action-independent drifting baseline and \(r(h,a)\) is the action-sensitive component. Define \(s(x,a):=\mathbb E_{\rm pool}\!\left[r(H,a)\mid X=x\right],\) the action-sensitive response averaged within representation cell \(x\). This is the action ranking that pooled prediction would need to preserve for deployment.

\begin{proposition}[Pooling can reverse the deployable action ranking]
\label{prop:misattribution_pooled}
Under Assumption~\ref{ass:one_step_identifiability} and the decomposition above,
the pooled predictive target admits \(q^\star(x,a)=s(x,a)+\delta(x,a),\) where \(\delta(x,a)\) is a drift-induced distortion caused by conditioning on the logged action. If \(\delta(x,\cdot)\) varies across actions, the ranking induced by \(q^\star(x,\cdot)\) can differ from that induced by \(s(x,\cdot)\).
In particular, there exist data-generating processes such that
\[
s(x,a)>s(x,a')
\qquad\text{but}\qquad
q^\star(x,a)<q^\star(x,a').
\]
\end{proposition}

Appendix~\ref{app:misattribution_pooled} gives the full expression for
\(\delta\), the proof, and a concrete ranking-reversal construction. Thus,
within-period identification alone does not justify optimizing a pooled
predictive surface: pooled regression may fit historical associations while selecting the wrong action for deployment. 

\section{Deployable Action Geometry}
\label{sec:stable_reduction}

Section~\ref{sec:problem_setup} shows that pooled prediction can learn the wrong action ranking under temporal drift. A natural candidate is an invariant response law. Such a law is both too strong and unnecessary for deployment: outcome levels may drift with demand, traffic, or policy conditions even when the relative ordering of actions is stable, while a one-step decision rule only requires the geometry of action contrasts, not the fully calibrated outcome surface.

We therefore target \emph{deployable causal action geometry}: the action-response shape that remains after removing action-independent offsets and positive amplitude variation on a suitable monotone link scale. Concretely, we anchor the response at a reference action to remove the linked-scale offset and profile out a positive scale factor, leaving the shared shape that determines action rankings. Thus, outcome levels and response amplitudes may vary across periods while the decision-relevant geometry remains transportable. Without a persistence assumption, observed periods cannot identify an unrestricted future response law; Appendix~\ref{app:future_not_identified} formalizes this point. We therefore treat this geometry as a restricted, decision-sufficient target for deployment under approximate persistence. See Appendix~\ref{app:terminology} for a terminology map.

Let \(\mathcal A_\star\subseteq\mathcal A\) be a compact common-support action region and \(a_0\in\mathcal A_\star\) a reference action. Assume \(a\mapsto m_t(h,a)\) is continuous on \(\mathcal A_\star\) for every observed period \(t\) and almost every \(h\).

\begin{figure}
    \centering
    \includegraphics[width=0.7\linewidth]{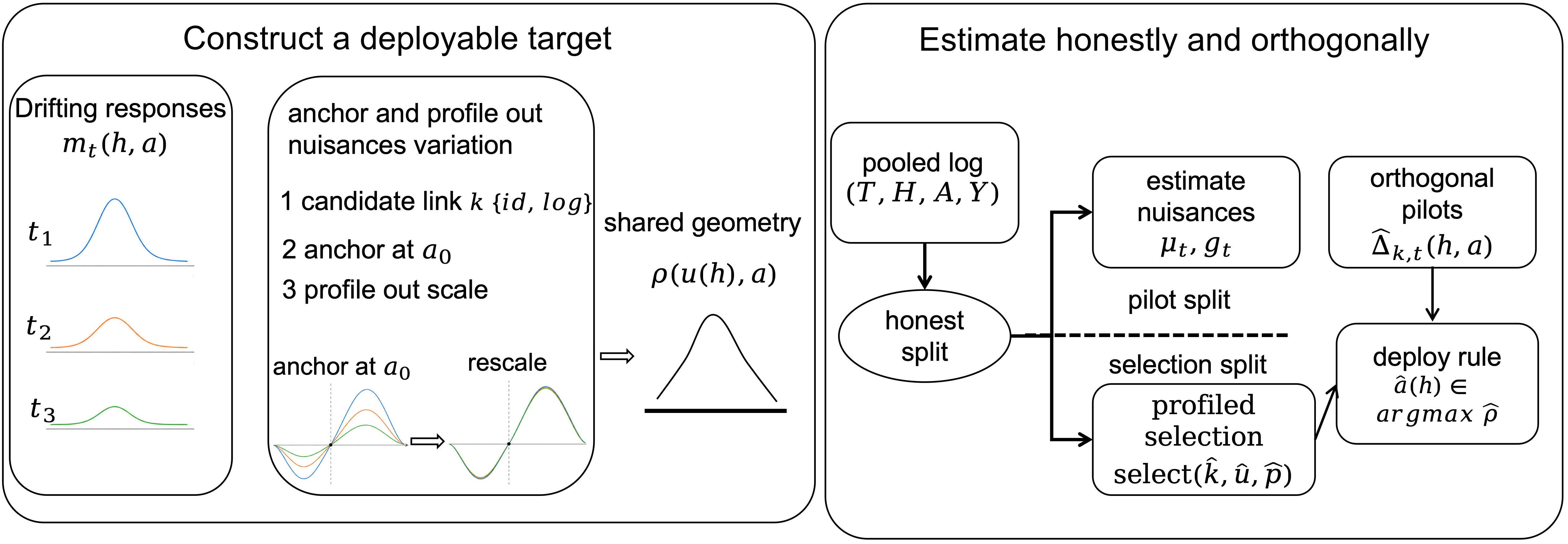}
    \caption{From drifting responses to deployable action geometry.}
    \label{fig:method}
\end{figure}

\subsection{Anchored link scales}
\label{subsec:linked_geometry}

For a strictly increasing link \(\ell_k\), define the anchored linked response \(\Delta_{k,t}(h,a):=\ell_k(m_t(h,a))-\ell_k(m_t(h,a_0)).\) Then \(\Delta_{k,t}(h,a_0)=0\), and anchoring removes an action-independent calibration term while preserving action rankings on the link scale.

We use a finite family of monotone links \(\{\ell_k\}_{k\in\mathcal K}\). Unless stated otherwise, \(\mathcal K=\{0,1\}\), with \(\ell_0(y)=y\) and \(\ell_1(y)=\log y\), where the log link is used only when \(m_t(h,a)>0\) on \(\mathcal A_\star\). The identity link removes additive calibration in anchored contrasts, while the log link removes multiplicative calibration. The following result formalizes their canonical role.

\begin{proposition}[Canonical exact nuisance-eliminating links]
\label{prop:canonical_links}
Let \(I\subseteq\mathbb R\) be an interval and let \(\ell:I\to\mathbb R\) be continuously differentiable and strictly increasing. 
(i) If \(\ell(y+b)-\ell(y_0+b)=\ell(y)-\ell(y_0)\) for all admissible \(y,y_0,b\), then \(\ell(y)=\alpha y+\beta\) on \(I\), with \(\alpha>0\). 
(ii) If \(I\subseteq(0,\infty)\) and \(\ell(cy)-\ell(cy_0)=\ell(y)-\ell(y_0)\) for all admissible
\(y,y_0,c>0\), then \(\ell(y)=\alpha\log y+\beta\) on \(I\), with \(\alpha>0\).
\end{proposition}

Thus identity and log are not merely preprocessing choices: if anchored contrasts must exactly remove additive or multiplicative calibration, they are canonical up to affine rescaling. Different links define different equivalence classes of response surfaces modulo time-varying calibration.

\subsection{Shared action geometry}
\label{subsec:shared_geometry}

The deployable object is not a period-specific anchored curve, but the part of that curve that persists after calibration. We seek a link \(k\), representation \(u\), and normalized shape \(\rho\) such that
\[
    \Delta_{k,t}(h,a)\approx s_t(h)\rho(u(h),a),\qquad s_t(h)\ge0.
    \label{eq:shared_geometry}
\]
Where \(s_t(h)\) absorbs period-specific calibration, while \(\rho(u(h),a)\) carries the action geometry for deployment.

Let \(\mathfrak U\) be a candidate class of representations \(u:\mathcal H\to\mathbb R^{d_u}\). For each \(k,u\), let \(\mathfrak R_k(u)\) be a normalized class of shapes \(\rho:u(\mathcal H)\times\mathcal A_\star\to\mathbb R\) with \(\rho(z,a_0)=0\); the scale/sign convention is in Appendix~\ref{app:normalization}. A candidate \((k,u,\rho)\) induces the action geometry
\[
    G_{k,u,\rho}(h):=\arg\max_{a\in\mathcal A_\star}\rho(u(h),a).
\]
When the maximizer is not unique, fix an arbitrary measurable selector. The latent tuple need not be unique; the deployment target is the induced action correspondence, or its decision-equivalence class (Appendix~\ref{app:representation_nonuniqueness}).

\subsection{Oracle target}
\label{subsec:oracle_target}
Let \(\varpi\) be a reference measure on \(\mathcal A_\star\) and write \(\|f\|_\varpi^2:=\int_{\mathcal A_\star} f(a)^2\,\varpi(da)\). For a candidate \((k,u,\rho)\), define the profiled geometry loss
\[
\mathcal A_k(u,\rho):=
\sum_{t=1}^T \pi_t\,
\mathbb E_t\!\left[
\inf_{s\ge 0}
\frac{
\left\|\Delta_{k,t}(H_t,\cdot)-s\rho(u(H_t),\cdot)\right\|_\varpi^2
}{
\left\|\Delta_{k,t}(H_t,\cdot)\right\|_\varpi^2
}
\right],
\label{eq:profiled_geometry_loss}
\]
with zero contribution when the denominator is zero. The infimum profiles out period-specific calibration, while the normalization removes amplitude, so candidates are compared by geometry rather than outcome level or link-scale magnitude. Appendix~\ref{app:profiled_calibration} gives the closed form of the profiled calibration, and Appendix~\ref{app:no_canonical_contrast} explains why fixing the anchored scale is necessary.

With regularization,
\(
    \mathcal L_k(u,\rho)
    :=
    \mathcal A_k(u,\rho)
    +
    \lambda_u\Omega_u(u)
    +
    \lambda_\rho\Omega_{k,\rho}(\rho).
\)
Any minimizer
\(
(k^\star,u^\star,\rho^\star)
\in
\arg\min_{k\in\mathcal K,\;u\in\mathfrak U,\;\rho\in\mathfrak R_k(u)}
\mathcal L_k(u,\rho)
\)
defines an oracle target within the admissible candidate family, with induced oracle geometry
\[
    G^\star(h):=
    \arg\max_{a\in\mathcal A_\star}\rho^\star(u^\star(h),a).
\]
Under exact calibrated-shape specification, \(\mathcal A_k(u,\rho)=0\) for a correct candidate; under misspecification, the target is the best regularized approximation within the admissible family. The object of interest is \(G^\star\), not a unique representation or a universal raw response surface.

\subsection{Deployability under approximate geometry persistence}
\label{subsec:approximate_deployability}

The oracle geometry is defined from observed periods. For deployment, we do not assume that the full future response surface is identified; we assume only that the anchored geometry persists approximately after calibration.

\begin{assumption}[Approximate geometry persistence]
\label{ass:approximate_persistence_main}
For deployment periods
\(\mathcal T_{\mathrm{dep}}\supseteq\{1,\ldots,T\}\), suppose that for every
\(t\in\mathcal T_{\mathrm{dep}}\), almost every \(h\), and every
\(a\in\mathcal A_\star\),
\[
    \Delta_{k^\star,t}(h,a)
    =
    s_t(h)\rho^\star(u^\star(h),a)+\xi_t(h,a),
    \qquad s_t(h)>0.
\]
\end{assumption}

The residual \(\xi_t\) measures violation of shared-geometry persistence. Small violation is sufficient for near-optimal action selection.

\begin{theorem}[Near-optimality under geometry persistence]
\label{thm:approximate_deployability_main}
Suppose Assumption~\ref{ass:approximate_persistence_main} holds. Let
\(a^\rho(h)\in G^\star(h)\) and
\(a_t^{\mathrm{opt}}(h)\in\arg\max_{a\in\mathcal A_\star}m_t(h,a)\). Then for
every \(t\in\mathcal T_{\mathrm{dep}}\) and almost every \(h\),
\[0\le
\Delta_{k^\star,t}\!\left(h,a_t^{\mathrm{opt}}(h)\right)
-
\Delta_{k^\star,t}\!\left(h,a^\rho(h)\right)
\le
2\sup_{a\in\mathcal A_\star}|\xi_t(h,a)|.
\]
If \(\xi_t\equiv0\), every selector from \(G^\star\) is optimal in every deployment period; otherwise, linked-response suboptimality is controlled by the persistence residual.
\end{theorem}

Because \(\ell_{k^\star}\) is strictly increasing and anchoring subtracts an action-independent constant, the maximizers of \(\Delta_{k^\star,t}(h,\cdot)\), \(\ell_{k^\star}(m_t(h,\cdot))\), and \(m_t(h,\cdot)\) coincide on the link-valid region. Thus the theorem connects the calibration-quotiented geometry to the original action decision. Section~\ref{sec:orthogonal_estimation} develops an orthogonal procedure for estimating this oracle target.

\section{Orthogonal Estimation of Deployable Action Geometry}
\label{sec:orthogonal_estimation}

Section~\ref{sec:stable_reduction} defines the population target through a profiled criterion over link, representation, and shape classes. We estimate this target with an honest two-stage procedure. Stage I constructs cross-fitted orthogonal pilots of the anchored linked responses. Stage II uses an independent selection sample to choose the link, representation, and action geometry by minimizing an empirical profiled loss. The resulting decision rule is induced by the selected geometry rather than by a globally stable raw response surface. The estimator below is the canonical low-dimensional construction analyzed theoretically; Section~\ref{sec:real_world_application} describes the scalable realization used in the Taobao deployment.

Let \(O=(\tau,H,A,Y)\) denote a pooled observation, where \(\tau\) is the period label, \(\mathbb P(\tau=t)=\pi_t>0\), and \(Y=Y_{\tau+1}\) conditional on \(\tau\). We split the pooled sample into a pilot sample \(\mathcal I^{\mathrm{pil}}\) and an independent selection sample \(\mathcal I^{\mathrm{sel}}\), preserving temporal proportions. This separation prevents the criterion used for selection from reusing the data that constructed its response pilots.

\subsection{Estimator}
\label{subsec:orthogonal_estimator}

\paragraph{Orthogonal anchored-response pilots.}
For each period \(t\), let \(\mu_{0,t}(h,a)=\mathbb E[Y\mid H=h,A=a,\tau=t]\) and \(g_{0,t}(a\mid h)=f_{A\mid H,\tau=t}(a\mid h)\). For a candidate representation \(u\), representation kernel \(L_\lambda\), and action kernel \(K_b\), define the localized orthogonal score
\begin{equation}
\label{eq:orthogonal_score}
\begin{aligned}
\phi_{t,\lambda,b}(O;u,h,a,\eta_t)
={}&
\frac{\mathbf 1\{\tau=t\}}{\pi_t}
L_\lambda\!\left(u(H)-u(h)\right)
\\[-2pt]
&\times
\Bigg[
\frac{K_b(A-a)}{g_t(A\mid H)}
\{Y-\mu_t(H,A)\}
+
\int_{\mathcal A_\star}
K_b(s-a)\mu_t(H,s)\,ds
\Bigg].
\end{aligned}
\end{equation}
where \(\eta_t=(\mu_t,g_t)\). The corresponding local mass score is \(\psi_{t,\lambda}(O;u,h)=\pi_t^{-1}\mathbf 1\{\tau=t\}L_\lambda(u(H)-u(h))\). The residualized term in~\eqref{eq:orthogonal_score} makes the pilot locally robust to first-order nuisance error, while the integral term targets the kernel-smoothed response. Identification and orthogonality are proved in Appendix~\ref{app:sec4_scores}; regularity conditions are stated in Appendix~\ref{app:sec4_stage1}.

Using cross-fitted nuisance estimates on \(\mathcal I^{\mathrm{pil}}\), let \(\hat\nu_{t,\lambda,b}(u,h,a)\) and \(\hat\delta_{t,\lambda}(u,h)\) be the empirical averages of \(\phi_{t,\lambda,b}\) and \(\psi_{t,\lambda}\), respectively. With a vanishing truncation threshold \(\epsilon_n>0\), define \(\hat m_{t,\lambda,b}(u,h,a)=\hat\nu_{t,\lambda,b}(u,h,a)/
\{\hat\delta_{t,\lambda}(u,h)\vee\epsilon_n\}\). For a link \(\ell_k\), the anchored pilot is \(\hat\Delta_{k,t}(u,h,a)= \ell_k^{(\epsilon_n)}\{\hat m_{t,\lambda,b}(u,h,a)\} -\ell_k^{(\epsilon_n)}\{\hat m_{t,\lambda,b}(u,h,a_0)\}\), where \(\ell_0^{(\epsilon_n)}(y)=y\) and \(\ell_1^{(\epsilon_n)}(y)=\log(y\vee\epsilon_n)\). The truncation stabilizes the estimated log pilot; the population log link remains defined on the link-valid positive region.

Only finite-grid control is needed for the selector. For the Stage-II action grid \(\{a^{(m)}\}_{m=1}^M\subset\mathcal A_\star\), Appendix~\ref{app:sec4_stage1} proves that, under the stated conditions,
\begin{equation}
\label{eq:stage1_grid_bound}
\sup_{k\in\mathcal K}\sup_{u\in\mathfrak U_n}\sup_{1\le t\le T}
\max_{i\in\mathcal I^{\mathrm{sel}}}\max_{1\le m\le M}
\left|
\hat\Delta_{k,t}(u,H_i,a^{(m)})
-
\Delta_{k,t}(H_i,a^{(m)})
\right|
=
O_p(\varepsilon_n^{\mathrm{pilot}}).
\end{equation}
For a fixed evaluation point, the corresponding rate is \((n_t\lambda^{d_u}b)^{-1/2}+\bar r_n+\lambda^{\alpha_u}+b^{\alpha_a}\), where \(\bar r_n\) is the second-order nuisance product error; the uniform sequence \(\varepsilon_n^{\mathrm{pilot}}\) additionally absorbs the finite-library and finite-grid union step.

\paragraph{Honest profiled selection.}
Let \(\mathfrak U_n=\{u_1,\ldots,u_{M_n}\}\) be a finite representation library. For each \(k\in\mathcal K\) and \(u\in\mathfrak U_n\), let \(\mathfrak R_{k,n}(u)\) be a normalized shape class. Given quadrature points and weights \(\{(a^{(m)},w_m)\}_{m=1}^M\), write \(\|f\|_{M,w}^2=\sum_{m=1}^M w_m f(a^{(m)})^2\). For \(i\in\mathcal I^{\mathrm{sel}}\), let \(\hat\Delta^u_{k,i}(\cdot)=\hat\Delta_{k,\tau_i}(u,H_i,\cdot)\) and \(\rho_i^u(\cdot)=\rho(u(H_i),\cdot)\). The empirical profiled geometry loss is
\begin{equation}
\label{eq:empirical_profiled_loss}
\hat{\mathcal A}^{\mathrm{sel}}_k(u,\rho)
=
\frac{1}{n_{\mathrm{sel}}}
\sum_{i\in\mathcal I^{\mathrm{sel}}}
\inf_{s\ge0}
\frac{
\|\hat\Delta^u_{k,i}-s\rho_i^u\|_{M,w}^2
}{
\|\hat\Delta^u_{k,i}\|_{M,w}^2
},
\end{equation}
with the zero-energy convention of Section~\ref{subsec:oracle_target}. The closed-form profiler, denominator stability, and quadrature approximation are given in Appendix~\ref{app:sec4_profiled}. With penalties \(\Omega_u\) and
\(\Omega_{k,\rho}\), define
\(\hat{\mathcal L}^{\mathrm{sel}}_k(u,\rho)=
\hat{\mathcal A}^{\mathrm{sel}}_k(u,\rho)
+\lambda_{u,n}\Omega_u(u)+\lambda_{\rho,n}\Omega_{k,\rho}(\rho)\), and select
\[
(\hat k,\hat u,\hat\rho)
\in
\arg\min_{k\in\mathcal K,\;u\in\mathfrak U_n,\;\rho\in\mathfrak R_{k,n}(u)}
\hat{\mathcal L}^{\mathrm{sel}}_k(u,\rho).
\]
The deployed rule is \(\hat a(h)\in\arg\max_{a\in\mathcal A_\star}\hat\rho(\hat u(h),a)\). The normalization in~\eqref{eq:empirical_profiled_loss} is essential: it compares candidates by action geometry rather than by the numerical amplitude of their anchored contrasts.

\subsection{Finite-library guarantee}
\label{subsec:main_stat_result}

The main guarantee is criterion-level learnability within the admissible finite library. Let \(p_n\) be the largest finite-dimensional shape parameter dimension, let \(q_M\) denote quadrature error, and define \(r_n=\varepsilon_n^{\mathrm{pilot}}+ \sqrt{(\log|\mathfrak U_n|+p_n\log n)/n_{\mathrm{sel}}}+q_M\).

\begin{theorem}[Finite-library oracle inequality]
\label{thm:main_learnability}
Assume the pilot bound~\eqref{eq:stage1_grid_bound} holds, together with the finite-library entropy, quadrature, energy-separation, and penalty conditions in Appendix~\ref{app:sec4_stage2}. Then
\begin{equation}
\label{eq:finite_library_oracle}
\mathcal L_{\hat k}(\hat u,\hat\rho)
-
\inf_{k\in\mathcal K,\;u\in\mathfrak U_n,\;\rho\in\mathfrak R_{k,n}(u)}
\mathcal L_k(u,\rho)=O_p(r_n).
\end{equation}
\end{theorem}

Theorem~\ref{thm:main_learnability} is deliberately criterion-level. It states that, within the chosen link, representation, and shape classes, the selected candidate approaches the best admissible population criterion at rate \(r_n\). It does not by itself imply recovery of a unique latent representation or policy regret in a future period.

Two conditional consequences connect criterion learnability to link recovery and deployment. First, if \(k^\star\) is the unique population link minimizer and the link selection gap \(\gamma_n=\min_{k\ne k^\star}\{V_n(k)-V_n(k^\star)\}\), with \(V_n(k)=\inf_{u,\rho}\mathcal L_k(u,\rho)\), satisfies \(r_n=o(\gamma_n)\), then \(\mathbb P(\hat k=k^\star)\to1\). Second, under the criterion-to-geometry bridge in Appendix~\ref{app:sec4_policy}, approximate geometry persistence in the deployment period, and a standard margin condition, the induced policy satisfies \(\mathrm{Regret}_t(\hat a)= O_p(\Psi_n(r_n)^{\kappa+1})+O(\mathrm{Pers}_t)\), where \(\mathrm{Pers}_t\) measures violation of shared geometry in period \(t\). Thus statistical error enters through the oracle-criterion rate \(r_n\), while irreducible deployment error is exactly the cost of persistence failure.

\section{Synthetic Experiments}
\label{sec:Synthetic_Experiments}
We use synthetic experiments to examine link recovery, future-period
decision quality, and the contributions of link selection and orthogonal
estimation. We consider four regimes. \(R1\) has additive temporal shifts,
for which the identity link is correct; \(R2\) has multiplicative shifts,
for which the log link is correct. \(R3\) introduces approximate action-dependent drift, and \(R4\) deliberately violates stable geometry through strong time-varying shape drift. Unless otherwise stated, we use \(T=20\) periods, a \(12/4/4\) train/validation/test split, a 51-point action grid, and 30 Monte Carlo repetitions. Our primary metrics are future-period stable IMSE (SIMSE) and regret; factual RMSE is reported only as a diagnostic. Full DGP details, implementation, and additional stress tests are in Appendix~\ref{app:sim_additional}.

\paragraph{Main evidence.}
Figure~\ref{fig:synthetic_main}(a,b) examines link recovery in the two
exact regimes, with recovery generally improving as sample size increases. Representation recovery is harder, but the deployable object is the induced action geometry rather than the latent representation itself. Figure~\ref{fig:synthetic_main}(c,d) and Table~\ref{tab:future_main} show lower future-period SIMSE and regret than pooled prediction and decision performance close to the regime-matched fixed-link reference. The close regret values indicate low adaptation cost when the generating link must be selected.
After fixing the selected link-representation pair, the orthogonal refit yields lower regret and action error than direct plug-in in the reported nuisance stress tests. For component attribution, an additional \(2\times2\) ablation varies orthogonal versus controlled plug-in pilots, which omit only the residual correction, at selection and refitting. Under outcome-model misspecification, most regret reduction occurs at refitting, while selection-stage effects on link recovery are regime-dependent. Full ablations are reported in Appendix~\ref{app:sim_orthogonal}.

\begin{figure}[t]
\centering
\begin{minipage}[t]{0.45\linewidth}
    \centering
    \includegraphics[width=\linewidth]{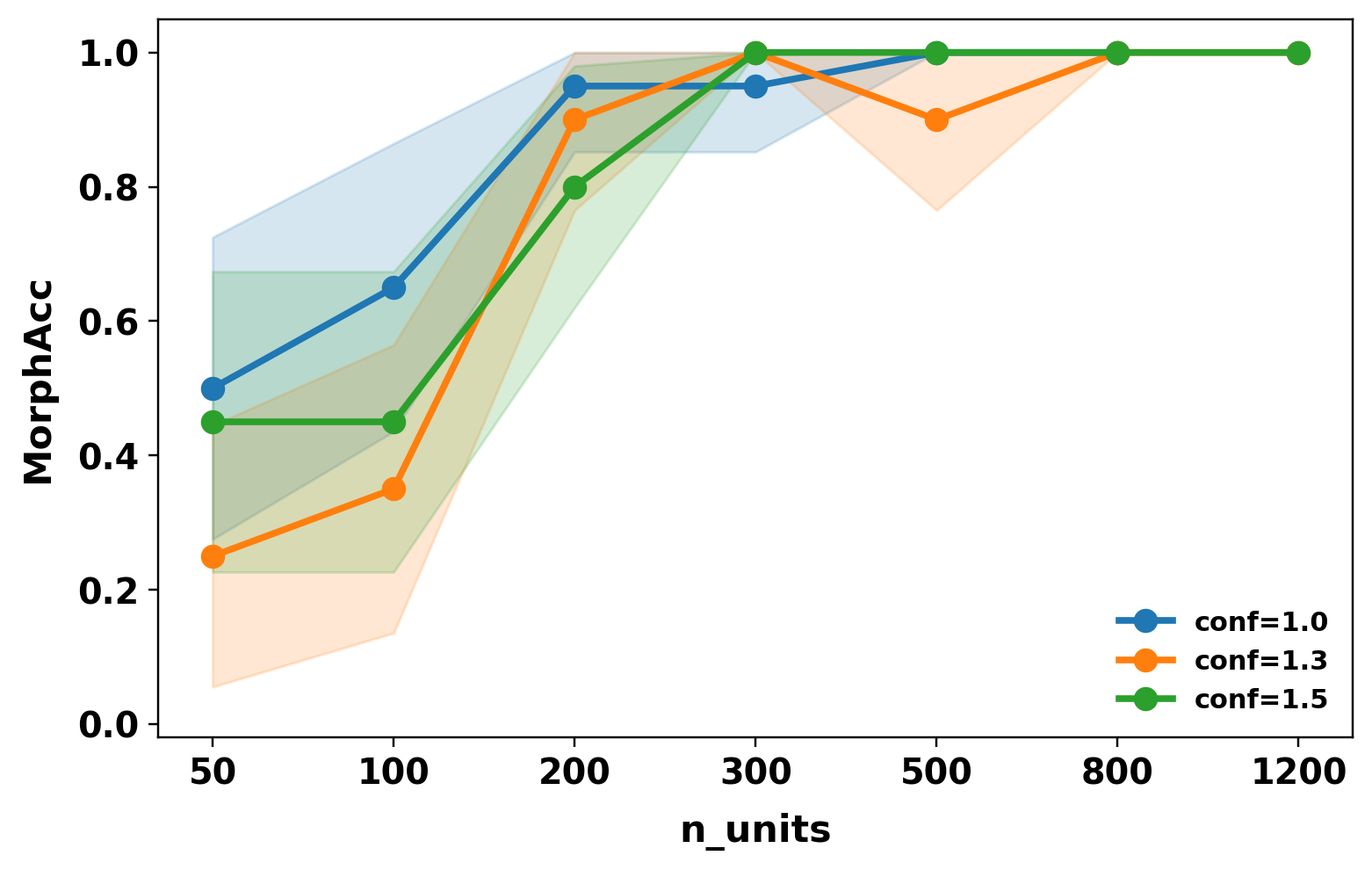}
    \vspace{-0.7em}
    {\footnotesize\textbf{(a)} Link recovery in R1.}
\end{minipage}
\hfill
\begin{minipage}[t]{0.45\linewidth}
    \centering
    \includegraphics[width=\linewidth]{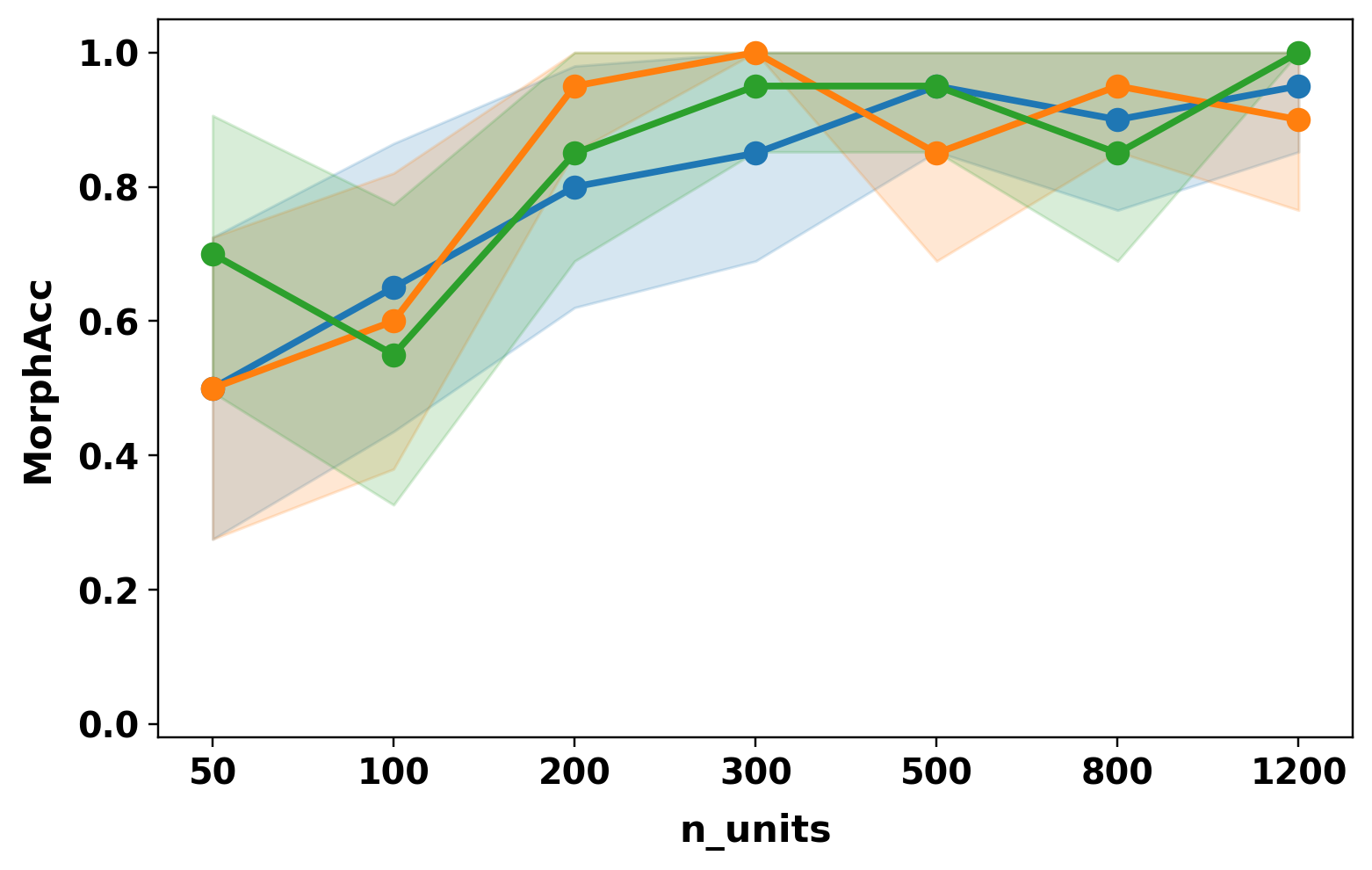}
    \vspace{-0.7em}

    {\footnotesize\textbf{(b)} link recovery in R2.}
\end{minipage}

\vspace{0.35em}

\begin{minipage}[t]{0.45\linewidth}
    \centering
    \includegraphics[width=\linewidth]{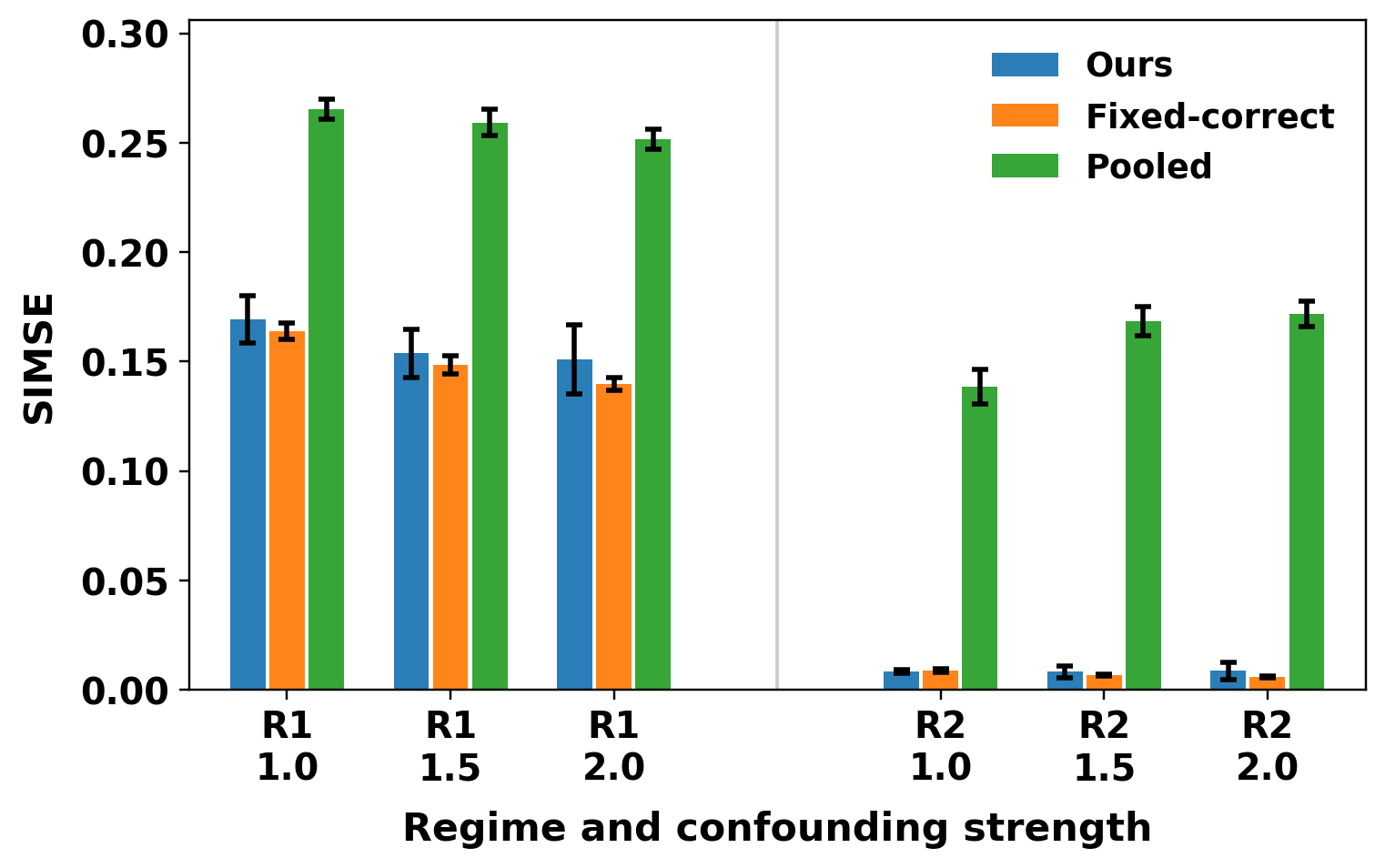}
    \vspace{-0.7em}

    {\footnotesize\textbf{(c)} Future-time SIMSE.}
\end{minipage}
\hfill
\begin{minipage}[t]{0.45\linewidth}
    \centering
    \includegraphics[width=\linewidth]{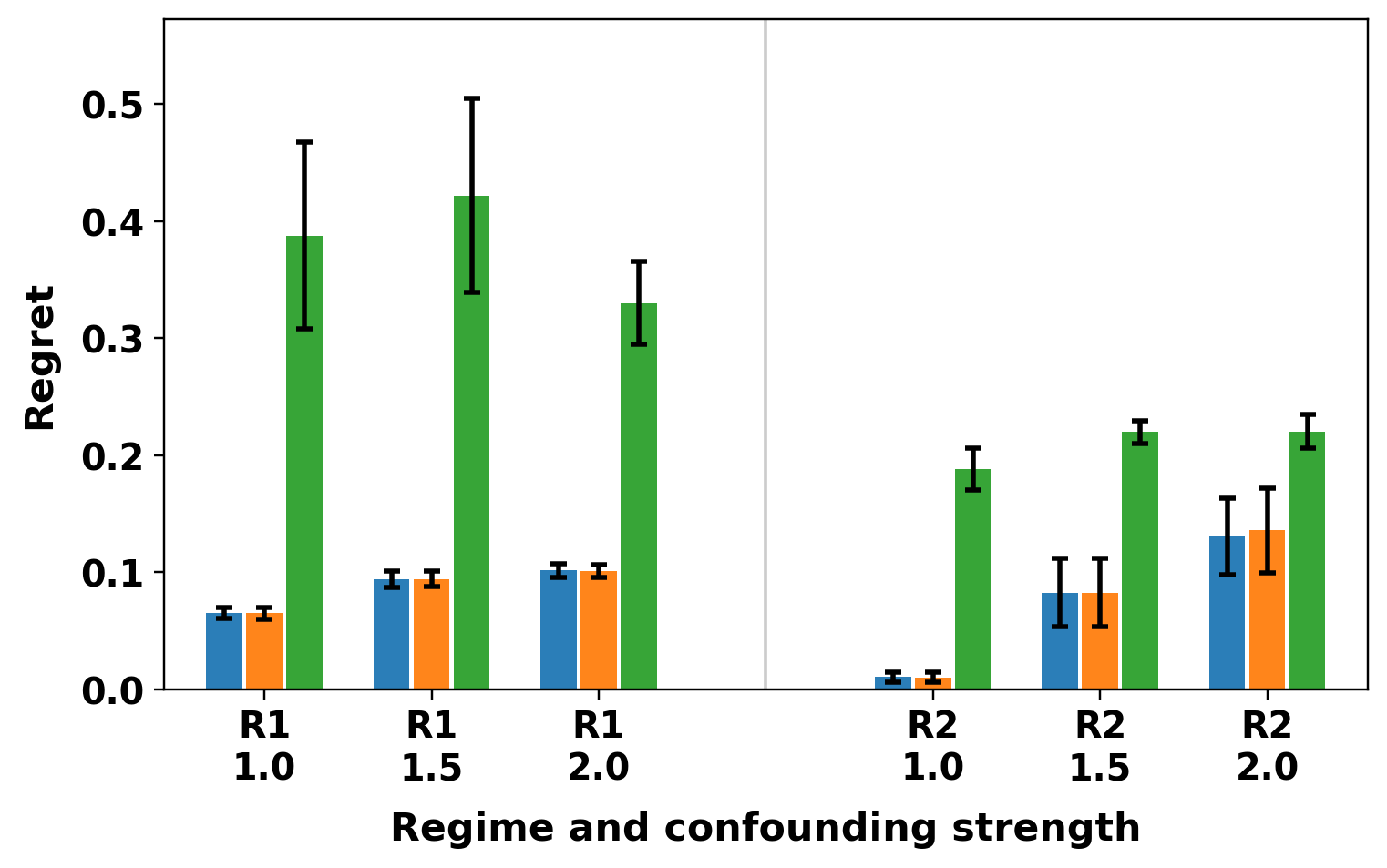}
    \vspace{-0.7em}

    {\footnotesize\textbf{(d)} Future-time regret.}
\end{minipage}

\vspace{-0.35em}
\caption{\textbf{Synthetic evidence.} Top: link recovery in the exact regimes. Bottom: future-period SIMSE and regret.}
\label{fig:synthetic_main}
\vspace{-0.6em}
\end{figure}

\begin{table}[t]
\centering
\setlength{\tabcolsep}{3.5pt}
\small
\caption{
\textbf{Future-period deployment.}
Each cell reports SIMSE / regret. Fixed-correct denotes the regime-matched fixed link: raw in \(R1\) and log in \(R2\).
}
\label{tab:future_main}
\begin{tabular}{llccc}
\toprule
Regime & Conf. & Ours & Fixed-correct & Pooled \\
\midrule
\multirow{3}{*}{\(R1\)}
& 1.0
& \(0.169{\pm}0.011\) / \(0.065{\pm}0.005\)
& \(0.164{\pm}0.004\) / \(0.065{\pm}0.005\)
& \(0.265{\pm}0.005\) / \(0.388{\pm}0.080\) \\
& 1.5
& \(0.154{\pm}0.011\) / \(0.094{\pm}0.007\)
& \(0.149{\pm}0.004\) / \(0.094{\pm}0.007\)
& \(0.259{\pm}0.006\) / \(0.422{\pm}0.083\) \\
& 2.0
& \(0.151{\pm}0.016\) / \(0.102{\pm}0.006\)
& \(0.140{\pm}0.003\) / \(0.101{\pm}0.006\)
& \(0.252{\pm}0.005\) / \(0.330{\pm}0.035\) \\
\midrule
\multirow{3}{*}{\(R2\)}
& 1.0
& \(0.008{\pm}0.001\) / \(0.011{\pm}0.004\)
& \(0.009{\pm}0.001\) / \(0.010{\pm}0.004\)
& \(0.139{\pm}0.008\) / \(0.189{\pm}0.018\) \\
& 1.5
& \(0.008{\pm}0.003\) / \(0.083{\pm}0.029\)
& \(0.007{\pm}0.001\) / \(0.083{\pm}0.029\)
& \(0.169{\pm}0.007\) / \(0.220{\pm}0.009\) \\
& 2.0
& \(0.009{\pm}0.004\) / \(0.131{\pm}0.033\)
& \(0.006{\pm}0.000\) / \(0.136{\pm}0.036\)
& \(0.172{\pm}0.006\) / \(0.220{\pm}0.014\) \\
\bottomrule
\end{tabular}
\end{table}

\paragraph{When link selection matters.}
We additionally shift the history distribution while preserving exact
action geometry, so that a wrong fixed link induces curvature outside
the shared shape class. Ours attains NRegret \(0.1385\), close to the
reference given the generating link (\(0.1384\)), and below plug-in
selection (\(0.1412\)), fixed identity (\(0.1416\)), and fixed log
(\(0.1422\)). This isolates the benefit of selecting the response scale
under transport shift;

\paragraph{Boundary and sensitivity diagnostics.}
Appendix~\ref{app:sim_additional} further probes the assumptions behind the target. In \(R3\), increasing action-dependent drift weakens link recovery; an additional controlled sweep shows increasing oracle-geometry regret as the persistence residual grows. In \(R4\), no shared geometry exists, so the goal is failure-boundary characterization rather than uniform superiority. We also report overlap stress tests and sensitivity to the action grid and reference action.

\section{Real-world Pricing Dataset and Application}
\label{sec:real_world_application}
We evaluate our framework on a large-scale observational pricing dataset from Taobao, one of China's lagest e-commerce platforms. Historical prices reflect adaptive merchant behavior under non-stationary demand driven by traffic, seasonality, and campaigns. We construct an item-day panel where the action is the price normalized near $1$. Offline evaluation is conducted via a rolling-origin chronological split on three months of data, containing $\sim 546.3M $ observations across $8.3M$ unique items (averaging $\sim 6.0M$ per day). Full construction details are provided in Appendix~\ref{app:real_world}.

\paragraph{Production implementation.}
At Taobao scale, we deploy a scalable target aligned neural realization that preserves the same anchored identity/log estimand and honest two-stage separation. A fixed 320-dimensional multimodal encoder maps histories to representations. Separate baseline and treatment nuisance models are trained with item-disjoint cross-fitting, and all out-of-fold predictions are frozen before response fitting. The action-response shape is modeled by an anchored three-layer MLP, trained with an R-loss-style residualized objective. Because the encoder is fixed before link selection, the representation library is a singleton; at each weekly retraining window, identity- and log-link candidates are compared using the profiled criterion on an independent chronological validation period, and the selected checkpoint is exactly the one used for deployment.

\subsection{Offline Evaluation}
\noindent\textbf{Metrics and baseline} We report three primary offline metrics: relative factual MAE $RMAE$ and $MAPE$ on observed outcomes; policy value $V_{DR}$ induced by models evaluated with a localized DR estimator for continuous actions \citep{kallus2018policy}; $\mathrm{StabLoss}$ measures future time transportability. We use baseline like Neural Parametric Model (NPM) which is a strong baseline in the business application, VcNet, Causal Transformer, Pooled MLP, and Time-aware pooled MLP.
\textbf{Main offline results.} Table~\ref{tab:real_baseline} reports factual accuracy, future-profile stability, and DR-based offline policy value. The baselines exhibit a clear trade-off: Causal Transformer gives the best factual fit, fixed raw-anchoring gives the best stability, and NPM gives the strongest policy value. Our method dominates this trade-off, achieving the best accuracy, the lowest stability loss, and the highest DR Policy Value Index. These results support our claim that future-time pricing decisions require stable response learning rather than same-period prediction alone. 
\textbf{Link-scale diagnosis.}
Figure~\ref{fig:realworld_combined}(b) compares a split level future-profile inconsistency under the raw-scale and log-linked anchored contrasts. The log-linked contrast attains a lower score in four of the six rolling splits and also yields a lower mean score across splits ($0.5566$ vs.\ $0.5718$). The evidence is strongest in Splits 3 and 5, while Splits 1 and 6 favor the raw-scale contrast.We therefore interpret the evidence as time-varying preference between the two response scales. As intended, this diagnosis is a deployment-oriented structural check based on future-period profile consistency rather than a claim of ground-truth link recovery.
\textbf{Temporal drift and deployment distance.}
We visualize the non-stationary demand environment underlying the real-world pricing task in Figure~\ref{fig:realworld_combined}(a). The figure plots normalized market-level and category-level demand indices over time, together with the rolling evaluation windows used in our chronological protocol. The demand trajectories vary substantially across periods, confirming that the deployment environment is far from stationary. This motivates evaluating pricing policies chronologically rather than with random splits. 

\begin{table}[t]
\centering
\small
\caption{Offline evaluation.}
\label{tab:real_baseline}
\setlength{\tabcolsep}{5pt}
\begin{tabular}{lcccc}
\toprule
Method & RMAE $\downarrow$ & MAPE $\downarrow$ & Stability $\downarrow$ & DR Policy Value Index $\uparrow$ \\
\midrule
Logged / behavior   & --    & --    & --    & 100.00 \\
NPM                 & 0.176 & 18.27 &  --   & 121.72 \\
VCNet               & 0.385 & 48.85 & 0.452 & 116.09 \\
Causal Transformer  & 0.174 & 17.35 & 0.566 & 107.18 \\
Pooled predictive   & 0.337 & 44.69 & 0.503 & 89.43 \\
Time-aware pooled   & 0.496 & 57.91 & 0.433 & 72.96 \\
Fixed identity  & 0.223 & 26.85 & 0.323 & 117.41 \\
Fixed log          & 0.138 & 16.33 & 0.282 & 122.36 \\
Ours                & 0.103 & 12.75 & 0.259 & 128.32  \\
\bottomrule
\end{tabular}
\end{table}

\subsection{Online A/B Result}
We finally test whether the offline gains translate into real deployment. We run an online A/B experiment on $\sim 280$k items over 26 days (2026/03/24--2026/04/18). As shown in Figure~\ref{fig:realworld_combined}(c), the treatment policy yields consistent positive lifts on all three core business metrics: GMV (+1.16\%, $p=0.037$), Buy counts (+1.70\%, $p=0.004$), and Commission (+1.51\%, $p<0.001$). These results provide direct online evidence that the learned price-response structure improves decision quality beyond retrospective offline evaluation.

\begin{figure}[t]
\centering
\begin{minipage}[t]{0.54\linewidth}
    \vspace{0pt}
    \centering
    \includegraphics[width=\linewidth]{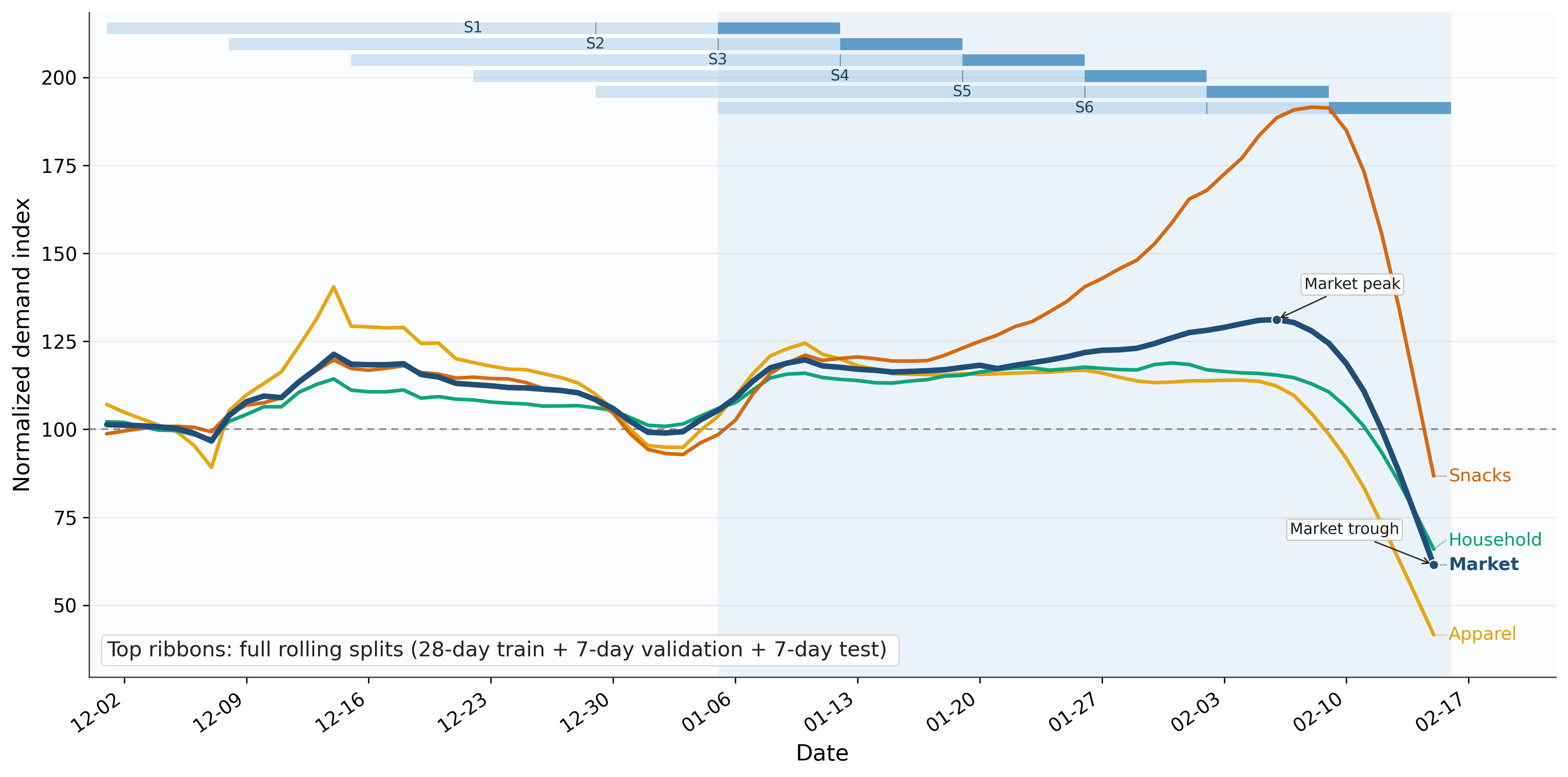}
    \vspace{-0.6em}

    {\footnotesize\textbf{(a)} Temporal drift and rolling windows.}
\end{minipage}
\hfill
\begin{minipage}[t]{0.43\linewidth}
    \vspace{0pt}
    \centering

    {\footnotesize\textbf{(b)} Link diagnosis.}
    \scriptsize
    \setlength{\tabcolsep}{2.3pt}
    \begin{tabular}{lrrrr}
    \toprule
    Split & Raw & Log & Pref. & Gap \\
    \midrule
    S1 & 0.6657 & 0.7311 & Raw & 8.95 \\
    S2 & 0.5946 & 0.5848 & Log & 1.64 \\
    S3 & 0.5183 & 0.4231 & Log & 18.36 \\
    S4 & 0.4704 & 0.4551 & Log & 3.26 \\
    S5 & 0.5087 & 0.4297 & Log & 15.54 \\
    S6 & 0.6733 & 0.7161 & Raw & 5.97 \\
    \midrule
    Mean & 0.5718 & 0.5566 & Log & 2.66 \\
    \bottomrule
    \end{tabular}

    \vspace{0.55em}

    {\footnotesize\textbf{(c)} Online A/B result.}
    \scriptsize
    \setlength{\tabcolsep}{4.2pt}
    \begin{tabular}{lcc}
    \toprule
    Metric & Lift (\%) & $p$ \\
    \midrule
    GMV & +1.16 & 0.037 \\
    Buy counts & +1.70 & 0.004 \\
    Commission & +1.51 & $<0.001$ \\
    \bottomrule
    \end{tabular}
\end{minipage}

\vspace{-0.4em}
\caption{\textbf{Real-world diagnostics and deployment.}
(a) temporal drift.
(b) future profile favors the log link while revealing heterogeneity.
(c) Online A/B testing shows positive lift.}
\label{fig:realworld_combined}
\vspace{-0.6em}
\end{figure}

\subsection{Public Dominic's Evaluation}
\label{sec:public_evaluation}
To complement the proprietary Taobao study, we evaluate our method on the public Dominick's retail dataset \citep{dominicks} using chronological splits. We assess factual prediction on the observed panel and counterfactual decision quality in semi-synthetic experiments that retain the real temporal and treatment-assignment design. These experiments cover additive and multiplicative persistence and controlled action-dependent shape drift. Under exact persistence, decision regret is close to that of a reference supplied with the generating link, indicating low adaptation cost. Appendix~\ref{app:dominicks} provides data preparation, evaluation protocols, outcome generators, implementation details, and full baseline comparisons.

\section{Related Work}
\label{sec:related_work}
Correcting confounding and identifying what can transfer are distinct problems. Continuous-treatment estimators and longitudinal counterfactual models address response estimation from observational histories \citep{kennedy2017continuous,nie2021vcnet,melnychuk2022causaltransformer}, while orthogonal learning reduces sensitivity to nuisance estimation \citep{chernozhukov2018double,foster2019orthogonal}. Neither causal identification nor nuisance robustness alone specifies which part of a historical response remains useful for future decisions.
Causal forecasting and invariant prediction make cross-environment transfer explicit \citep{schultz2023causalpricing,crasson2024causalforecasting,peters2016icp,arjovsky2019irm, liuopflow}. Our distinction is the object being transported: an anchored conditional action-response shape, rather than its period-specific level or amplitude. Under exact geometry persistence, environments may disagree about outcome magnitudes while agreeing on the preferred actions; approximate persistence quantifies departures from this structure. The proposed target therefore complements confounding adjustment by making the required transfer assumption explicit. Appendix~\ref{app:related_work} gives detailed comparisons.

\section{Conclusion and Limitations}
\label{sec:conclusion_limitations}
We studied continuous treatment-response learning under temporal non-stationarity, where pooled prediction can conflate action effects with drifting baselines and target an object that is unsuitable for deployment. We introduced deployable causal action geometry, a decision target that captures the action-response structure preserved after reference-action anchoring and link-scale recalibration. We developed an honest two-stage orthogonal estimator with profiled link and representation selection. Theory and experiments on synthetic data, the Dominick's dataset, and a large-scale pricing application support a central message: under temporal drift, decision learning should prioritize stable action geometry rather than raw prediction. Its validity relies on one-step exchangeability, overlap, a link family, and approximate geometry persistence; future work could relax the structure assumptions and extend the framework to long-horizon sequential decision making.

\section*{Acknowledgments}
This work was supported by the National Natural Science Foundation
of China (Grant No. 42571530) and by Alibaba Group through Alibaba Research Intern Program. We thank the anonymous reviewers and area chairs for their constructive feedback.

\medskip{
\small
\bibliographystyle{plainnat}
\bibliography{references}

@inproceedings{hua2021markdowns,
  title={Markdowns in e-commerce fresh retail: A counterfactual prediction and multi-period optimization approach},
  author={Hua, Junhao and Yan, Ling and Xu, Huan and Yang, Cheng},
  booktitle={Proceedings of the 27th ACM SIGKDD conference on knowledge discovery \& data mining},
  pages={3022--3031},
  year={2021},
  doi={10.1145/3447548.3467083}
}

@misc{multichanneluplift,
      title={Multi-channel Uplift Policy Learning}, 
      author={Changjian Liu and Tianyu Wang and Xiaoxuan Deng and WenTao Zhu and Yuwei Xu and Jungqi Jin and Yong Gao and Chuan Yu and Jian Xu and Bo Zheng},
      year={2026},
      eprint={2607.28182},
      archivePrefix={arXiv},
      primaryClass={cs.LG},
      url={https://arxiv.org/abs/2607.28182}, 
}

@misc{liuopflow,
      title={OpFlow: Learning Opportunity-Conditioned Choice Potentials for Robust OD Flow Prediction}, 
      author={Changjian Liu and Yong Gao and Yuqing Wang and Leyi Su and Honglei Guo and Zhiyang Wang and Xiaoyu Wang and Fan Zhang},
      year={2026},
      eprint={2607.03200},
      archivePrefix={arXiv},
      primaryClass={cs.LG},
      url={https://arxiv.org/abs/2607.03200}, 
}

@article{dominicks,
  title={Creating micro-marketing pricing strategies using supermarket scanner data},
  author={Montgomery, Alan L},
  journal={Marketing science},
  volume={16},
  number={4},
  pages={315--337},
  year={1997},
  publisher={INFORMS}
}

@article{athey2021policy,
  title   = {Policy learning with observational data},
  author  = {Athey, Susan and Wager, Stefan},
  journal = {Econometrica},
  volume  = {89},
  number  = {1},
  pages   = {133--161},
  year    = {2021},
  doi     = {10.3982/ECTA15732}
}

@article{kennedy2017continuous,
  title={Non-parametric methods for doubly robust estimation of continuous treatment effects},
  author={Kennedy, Edward H and Ma, Zongming and McHugh, Matthew D and Small, Dylan S},
  journal={Journal of the Royal Statistical Society Series B: Statistical Methodology},
  volume={79},
  number={4},
  pages={1229--1245},
  year={2017},
  doi={10.1111/rssb.12212}
}

@inproceedings{nie2021vcnet,
  title     = {VCNet and functional targeted regularization for learning causal effects of continuous treatments},
  author    = {Nie, Lizhen and Ye, Mao and Liu, Qiang and Nicolae, Dan},
  booktitle = {International Conference on Learning Representations},
  year      = {2021}
}

@incollection {hirano2005propensity,
    AUTHOR = {Hirano, Keisuke and Imbens, Guido W.},
     TITLE = {The propensity score with continuous treatments},
 BOOKTITLE = {Applied {B}ayesian modeling and causal inference from incomplete-data perspectives},
    SERIES = {Wiley Ser. Probab. Stat.},
     PAGES = {73--84},
 PUBLISHER = {Wiley, Chichester},
      YEAR = {2004},
   MRCLASS = {62-07},
  MRNUMBER = {2134803},
       DOI = {10.1002/0470090456.ch7},
}

@article{lim2018rmsn,
  title={Forecasting treatment responses over time using recurrent marginal structural networks},
  author={Lim, Bryan},
  journal={Advances in neural information processing systems},
  volume={31},
  year={2018}
}

@inproceedings{bica2020crn,
  title     = {Estimating counterfactual treatment outcomes over time through adversarially balanced representations},
  author    = {Bica, Ioana and Alaa, Ahmed M. and Jordon, James and van der Schaar, Mihaela},
  booktitle = {International Conference on Learning Representations},
  year      = {2020}
}

@inproceedings{frauen2025temporalmetalearners,
  title     = {Model-agnostic meta-learners for estimating heterogeneous treatment effects over time},
  author    = {Frauen, Dennis and Hess, Konstantin and Feuerriegel, Stefan},
  booktitle = {The Thirteenth International Conference on Learning Representations},
  year      = {2025}
}

@inproceedings{melnychuk2022causaltransformer,
  title     = {Causal transformer for estimating counterfactual outcomes},
  author    = {Melnychuk, Valentyn and Frauen, Dennis and Feuerriegel, Stefan},
  booktitle = {Proceedings of the 39th International Conference on Machine Learning},
  series    = {Proceedings of Machine Learning Research},
  volume    = {162},
  pages     = {15293--15329},
  publisher = {PMLR},
  year      = {2022}
}

@article{crasson2024causalforecasting,
  title   = {Training and evaluating causal forecasting models for time-series},
  author  = {Crasson, Thomas and Nabet, Yacine and L{\'e}cuyer, Mathias},
  journal = {arXiv preprint arXiv:2411.00126},
  year    = {2024},
  doi     = {10.48550/arXiv.2411.00126}
}

@article{schultz2023causalpricing,
  title   = {Causal forecasting for pricing},
  author  = {Schultz, Douglas and Stephan, Johannes and Sieber, Julian and Yeh, Trudie and Kunz, Manuel and Doupe, Patrick and Januschowski, Tim},
  journal = {arXiv preprint arXiv:2312.15282},
  year    = {2023},
  doi     = {10.48550/arXiv.2312.15282}
}

@article{peters2016icp,
  title   = {Causal inference by using invariant prediction: identification and confidence intervals},
  author  = {Peters, Jonas and B{\"u}hlmann, Peter and Meinshausen, Nicolai},
  journal = {Journal of the Royal Statistical Society Series B: Statistical Methodology},
  volume  = {78},
  number  = {5},
  pages   = {947--1012},
  year    = {2016},
  doi     = {10.1111/rssb.12167}
}

@article{arjovsky2019irm,
  title={Invariant risk minimization},
  author={Arjovsky, Martin and Bottou, L{\'e}on and Gulrajani, Ishaan and Lopez-Paz, David},
  journal={arXiv preprint arXiv:1907.02893},
  year={2019},
  doi={10.48550/arXiv.1907.02893}
}

@article{rothenhausler2021anchor,
  title={Anchor regression: Heterogeneous data meet causality},
  author={Rothenh{\"a}usler, Dominik and Meinshausen, Nicolai and B{\"u}hlmann, Peter and Peters, Jonas},
  journal={Journal of the Royal Statistical Society Series B: Statistical Methodology},
  volume={83},
  number={2},
  pages={215--246},
  year={2021},
  doi = {10.1111/rssb.12398}
}

@inproceedings{krueger2021out,
  title={Out-of-distribution generalization via risk extrapolation (rex)},
  author={Krueger, David and Caballero, Ethan and Jacobsen, Joern-Henrik and Zhang, Amy and Binas, Jonathan and Zhang, Dinghuai and Le Priol, Remi and Courville, Aaron},
  booktitle={International conference on machine learning},
  pages={5815--5826},
  year={2021},
  organization={PMLR},
}

@article{chernozhukov2018double,
    author = {Chernozhukov, Victor and Chetverikov, Denis and Demirer, Mert and Duflo, Esther and Hansen, Christian and Newey, Whitney and Robins, James},
    title = {Double/debiased machine learning for treatment and structural parameters},
    journal = {The Econometrics Journal},
    volume = {21},
    number = {1},
    pages = {C1-C68},
    year = {2018},
    doi = {10.1111/ectj.12097}
}

@article{foster2019orthogonal,
  title={Orthogonal statistical learning},
  author={Foster, Dylan J and Syrgkanis, Vasilis},
  journal={The Annals of Statistics},
  volume={51},
  number={3},
  pages={879--908},
  year={2023},
  publisher={Institute of Mathematical Statistics}
}

@inproceedings{kallus2018policy,
  title={Policy evaluation and optimization with continuous treatments},
  author={Kallus, Nathan and Zhou, Angela},
  booktitle={International conference on artificial intelligence and statistics},
  pages={1243--1251},
  year={2018},
  volume = {84},
  organization={PMLR}
}

@article{imai2004general,
  title={Causal inference with general treatment regimes: Generalizing the propensity score},
  author={Imai, Kosuke and Van Dyk, David A},
  journal={Journal of the American Statistical Association},
  volume={99},
  number={467},
  pages={854--866},
  year={2004},
  publisher={Taylor \& Francis},
  doi  = {10.1198/016214504000001187}
}

@inproceedings{zeng2021causal,
  title={Causal transfer random forest: Combining logged data and randomized experiments for robust prediction},
  author={Zeng, Shuxi and Bayir, Murat Ali and Pfeiffer III, Joseph J and Charles, Denis and Kiciman, Emre},
  booktitle={Proceedings of the 14th ACM International Conference on Web Search and Data Mining},
  pages={211--219},
  year={2021},
  doi={10.1145/3437963.3441722}
}

@inproceedings{bica2020deconfounder,
  title={Time series deconfounder: Estimating treatment effects over time in the presence of hidden confounders},
  author={Bica, Ioana and Alaa, Ahmed and Van Der Schaar, Mihaela},
  booktitle={International conference on machine learning},
  pages={884--895},
  volume    = {119},
  year={2020},
  organization={PMLR}
}

@article{frauen2025decision,
  title={Treatment effect estimation for optimal decision-making},
  author={Frauen, Dennis and Melnychuk, Valentyn and Schweisthal, Jonas and van der Schaar, Mihaela and Feuerriegel, Stefan},
  journal={arXiv preprint arXiv:2505.13092},
  year={2025}
}

@misc{robins2000marginal,
  title={Marginal structural models and causal inference in epidemiology},
  author={Robins, James M and Hernan, Miguel Angel and Brumback, Babette},
  journal={Epidemiology},
  volume={11},
  number={5},
  pages={550--560},
  year={2000},
  doi={10.1097/00001648-200009000-00011}
}

@article{ganin2016domain,
  title={Domain-adversarial training of neural networks},
  author={Ganin, Yaroslav and Ustinova, Evgeniya and Ajakan, Hana and Germain, Pascal and Larochelle, Hugo and Laviolette, Fran{\c{c}}ois and March, Mario and Lempitsky, Victor},
  journal={Journal of machine learning research},
  volume={17},
  number={59},
  pages={1--35},
  year={2016}
}

@article{magliacane2018domain,
  title={Domain adaptation by using causal inference to predict invariant conditional distributions},
  author={Magliacane, Sara and Van Ommen, Thijs and Claassen, Tom and Bongers, Stephan and Versteeg, Philip and Mooij, Joris M},
  journal={Advances in neural information processing systems},
  volume={31},
  year={2018}
}

@article{subbaswamy2020development,
  title={From development to deployment: dataset shift, causality, and shift-stable models in health AI},
  author={Subbaswamy, Adarsh and Saria, Suchi},
  journal={Biostatistics},
  volume={21},
  number={2},
  pages={345--352},
  year={2020},
  doi={10.1093/biostatistics/kxz041}
}

@inproceedings{kallus2020doubly,
  title={Doubly robust off-policy value and gradient estimation for deterministic policies},
  author={Kallus, Nathan and Uehara, Masatoshi},
  booktitle={Advances in Neural Information Processing Systems},
  volume={33},
  pages={10420--10430},
  year={2020}
}
}

\newpage
\appendix

This appendix collects related work, theoretical proofs and supplementary results for Sections~\ref{sec:problem_setup}--\ref{sec:orthogonal_estimation}, the implementation algorithm, and additional synthetic-experiment details.

\section{Related Work}
\label{app:related_work}

\paragraph{Continuous treatments, policy learning, and orthogonal estimation.}
Continuous-action causal learning builds on generalized propensity scores, dose-response estimation, and policy learning from observational data
\citep{imai2004general,hirano2005propensity,kennedy2017continuous,kallus2018policy,athey2021policy}.
Recent neural methods such as VCNet estimate smooth average dose-response curves for continuous treatments \citep{nie2021vcnet}, while continuous-action off-policy evaluation develops kernelized or doubly robust estimators for policy value and gradients \citep{kallus2018policy,kallus2020doubly}. Our estimator also uses the orthogonal/debiased learning principle: cross-fitted nuisance estimates enter through scores with second-order sensitivity to nuisance error \citep{chernozhukov2018double,foster2019orthogonal}. These works provide the estimation primitives for confounded continuous actions. Their target, however, is usually a static dose-response curve, a policy value, or a fixed causal parameter. They do not address the target-level question central to this paper: when outcome levels and logging policies drift over time, which part of a history-conditioned continuous response remains meaningful for future action ranking?

\paragraph{Longitudinal counterfactual modeling.}
Classical longitudinal causal inference studies time-varying treatments and confounding through marginal structural models and related weighting ideas \citep{robins2000marginal}. Modern sequence models extend this perspective with recurrent or transformer architectures: RMSN forecasts responses under planned treatment sequences, CRN balances treatment histories for counterfactual prediction, the Time Series Deconfounder targets multi-cause hidden confounding, and Causal Transformer improves sequential counterfactual modeling with attention-based architectures \citep{lim2018rmsn,bica2020crn,bica2020deconfounder,melnychuk2022causaltransformer}. These methods are highly relevant because they learn from observational time series. Their objective is nevertheless different: they aim to predict counterfactual outcomes over treatment paths or estimate heterogeneous effects over time. We instead study a one-step continuous decision problem under temporal non-stationarity, where the key issue is not trajectory prediction but the construction of an anchored action geometry that can be transported to future periods.

\paragraph{Invariant learning, causal transfer, and deployment shift.}
A separate literature asks what should remain stable across environments. Domain-adversarial training learns representations that obscure domain identity \citep{ganin2016domain}; invariant causal prediction, IRM, Anchor Regression, and REx seek stable predictors or risk profiles across environments and perturbations \citep{peters2016icp,arjovsky2019irm,rothenhausler2021anchor,krueger2021out}; causal domain adaptation identifies invariant conditional distributions under distribution shifts \citep{magliacane2018domain}; and deployment-oriented work formalizes dataset shift through a causal lens \citep{subbaswamy2020development}. In industrial prediction, Causal Transfer Random Forest combines logged data with a small randomized experiment to learn predictors robust to feature shifts \citep{zeng2021causal}. These works are conceptually close because they share the goal of stable transfer. The difference is the object being stabilized. Prior work typically stabilizes a predictor, conditional distribution, or robust risk. Our target is a causal response geometry for confounded continuous actions: invariance is imposed only after anchoring, link transformation, and profiling out time-varying calibration.

\paragraph{Causal forecasting and decision-oriented learning.}
Recent work argues that time-series models used for decisions should be evaluated by intervention- or policy-relevant criteria rather than only predictive accuracy. Causal forecasting for pricing combines causal estimation with forecasting architectures for off-policy price-response learning \citep{schultz2023causalpricing}, and recent causal-forecasting protocols emphasize training and evaluation under shifted decisions \citep{crasson2024causalforecasting}. More broadly, decision-oriented treatment-effect learning shows that estimators optimized for effect accuracy may be suboptimal for downstream policy quality \citep{frauen2025decision}. Our work shares this deployment-oriented motivation but focuses on an earlier target-definition problem: under temporal drift, the raw response surface itself may not be the right object. We therefore define a restricted action geometry that preserves the intervention rankings needed for deployment while allowing period-specific calibration to vary.

\paragraph{Summary.}
The closest prior ideas are continuous-treatment estimation, longitudinal confounding adjustment, orthogonal learning, and invariant prediction under shift. Our contribution is their combination at the level of the estimand: we define a deployable, restricted continuous-action response on anchored link scales, estimate it with orthogonal pilots and honest profiled selection, and require persistence only for the action-sensitive geometry rather than for the full raw outcome surface.

\section{Terminology}
\label{app:terminology}

\begin{table*}[t]
\centering
\small
\caption{Terminology and standard interpretations.}
\label{tab:terminology}
\begin{tabular}{@{}p{0.20\textwidth}p{0.27\textwidth}p{0.47\textwidth}@{}}
\toprule
\textbf{Term} & \textbf{Standard interpretation} & \textbf{Role in our framework} \\
\midrule

Link function $\ell_k$
&
Monotone response scale
&
Defines the scale on which action-independent temporal variation can be removed.
The main candidates are the identity and log links.
\\[2pt]

Anchored response
$\Delta_{k,t}(h,a)$
&
Reference-action contrast
&
$\ell_k\{m_t(h,a)\}-\ell_k\{m_t(h,a_0)\}$.
Anchoring removes an action-independent offset while preserving action rankings.
\\[2pt]

Positive amplitude
$s_t(h)$
&
Nuisance scale factor
&
Captures period- and history-specific changes in the magnitude of the anchored response.
It is profiled out when comparing candidate geometries.
\\[2pt]

Action geometry
$\rho(u(h),a)$
&
Normalized conditional action-response shape
&
The decision-relevant shape remaining after anchoring and positive-scale profiling;
it determines the action ranking through $\arg\max_a \rho(u(h),a)$.
\\[2pt]

Deployable causal action geometry
&
Decision-sufficient response equivalence class
&
An anchored action-response shape considered modulo linked-scale offsets
and positive rescaling, restricted to the chosen link, representation,
and shape classes.
\\[2pt]

Profiled geometry loss
&
Nuisance-amplitude-profiled discrepancy
&
Measures how well one shared action geometry explains period-specific
anchored responses after optimally rescaling their amplitudes.
\\[2pt]

Approximate geometry persistence
&
Cross-period transport condition
&
Requires future anchored responses to remain close to a positive rescaling
of the learned geometry; $\xi_t(h,a)$ measures the violation.
\\[2pt]

Decision equivalence
&
Equality of induced action rankings
&
Different latent representations or response parameterizations are equivalent
when they induce the same maximizing action.
\\

\bottomrule
\end{tabular}
\end{table*}

\section{Appendix for Section~\ref{sec:problem_setup}: Problem Setup}

\subsection{Proof of Within-Period Identification}
\label{app:proof_pointwise_identification}

\begin{proof}
Fix a period \(t\) and suppress the unit index. Let
\[
    \mu_t(h,a)
    :=
    \mathbb E\!\left[Y_{t+1}\mid H_t=h,A_t=a,\text{ period }t\right]
\]
be a version of the observed conditional mean. For any action \(a\) in the
overlap region, consistency gives \(Y_{t+1}=Y_{t+1}(a)\) on the event
\(\{A_t=a\}\). Therefore,
\[
\begin{aligned}
\mu_t(h,a)
&=
\mathbb E\!\left[Y_{t+1}(a)\mid H_t=h,A_t=a,\text{ period }t\right] \\
&=
\mathbb E\!\left[Y_{t+1}(a)\mid H_t=h,\text{ period }t\right] \\
&=
m_t(h,a),
\end{aligned}
\]
where the second equality uses one-step exchangeability and the last equality is
the definition of \(m_t\). Hence
\[
    m_t(h,a)
    =
    \mathbb E\!\left[Y_{t+1}\mid H_t=h,A_t=a,\text{ period }t\right]
\]
for \((H_t,A_t)\)-almost every \((h,a)\) in the relevant support.
\end{proof}

\subsection{Proof of Proposition~\ref{prop:misattribution_pooled}}
\label{app:misattribution_pooled}

\begin{proof}
Let \(\tau\) denote the period label under the pooled distribution. By consistency,
one-step exchangeability, and iterated expectation,
\[
\begin{aligned}
q^\star(x,a)
&=
\mathbb E_{\rm pool}[Y\mid X=x,A=a] \\
&=
\mathbb E_{\rm pool}[Y(a)\mid X=x,A=a] \\
&=
\mathbb E_{\rm pool}\!\left[
    \mathbb E\!\left[Y(a)\mid H,A,\tau\right]
    \mid X=x,A=a
\right] \\
&=
\mathbb E_{\rm pool}\!\left[
    \mathbb E\!\left[Y(a)\mid H,\tau\right]
    \mid X=x,A=a
\right] \\
&=
\mathbb E_{\rm pool}\!\left[
    m_\tau(H,a)
    \mid X=x,A=a
\right].
\end{aligned}
\]
Under the decomposition \(m_t(h,a)=b_t(h)+r(h,a)\), this becomes
\[
q^\star(x,a)
=
\mathbb E_{\rm pool}[b_\tau(H)\mid X=x,A=a]
+
\mathbb E_{\rm pool}[r(H,a)\mid X=x,A=a].
\]
Adding and subtracting
\[
s(x,a):=\mathbb E_{\rm pool}[r(H,a)\mid X=x]
\]
yields
\[
q^\star(x,a)=s(x,a)+\delta(x,a),
\]
where
\[
\delta(x,a)
:=
\mathbb E_{\rm pool}[b_\tau(H)\mid X=x,A=a]
+
\mathbb E_{\rm pool}[r(H,a)\mid X=x,A=a]
-
\mathbb E_{\rm pool}[r(H,a)\mid X=x].
\]
Therefore, for any \(a,a'\),
\[
q^\star(x,a)-q^\star(x,a')
=
\{s(x,a)-s(x,a')\}
+
\{\delta(x,a)-\delta(x,a')\}.
\]
Thus pooled prediction preserves all pairwise action rankings at \(x\) whenever
\(\delta(x,\cdot)\) is constant, and can change them when the distortion varies
with action. The concrete construction below shows that reversal can occur.
\end{proof}

\begin{corollary}[Ranking reversal under drift-aligned logging]
\label{cor:ranking_reversal}
Let \(\mathcal A=\{a_1,a_2\}\), and suppose \(H\equiv h\) and \(X\equiv x\).
Consider two periods with equal pooled weights and drifting baselines
\(b_1(h)=0\), \(b_2(h)=M\). Suppose the action-sensitive component ranks
\(a_1\) above \(a_2\):
\[
    r(h,a_1)>r(h,a_2).
\]
Let the logging policy be aligned with time:
\[
    \mathbb P(A=a_1\mid \tau=1)=1-\epsilon,
    \qquad
    \mathbb P(A=a_1\mid \tau=2)=\epsilon,
    \qquad
    0<\epsilon<\frac12,
\]
with \(\mathbb P(A=a_2\mid \tau=t)=1-\mathbb P(A=a_1\mid \tau=t)\). Then
\[
    q^\star(x,a_1)-q^\star(x,a_2)
    =
    \{r(h,a_1)-r(h,a_2)\}-(1-2\epsilon)M.
\]
Consequently, if
\[
    M>\frac{r(h,a_1)-r(h,a_2)}{1-2\epsilon},
\]
the pooled target ranks \(a_2\) above \(a_1\), although the deployable
action-sensitive component ranks \(a_1\) above \(a_2\).
\end{corollary}

\begin{proof}
Since the two periods have equal pooled weights, Bayes' rule gives
\[
    \mathbb P(\tau=2\mid A=a_1)=\epsilon,
    \qquad
    \mathbb P(\tau=2\mid A=a_2)=1-\epsilon .
\]
Because \(H\equiv h\) and \(X\equiv x\),
\[
\begin{aligned}
q^\star(x,a_1)
&= r(h,a_1)+\mathbb E[b_\tau(h)\mid A=a_1]
 = r(h,a_1)+\epsilon M,\\
q^\star(x,a_2)
&= r(h,a_2)+\mathbb E[b_\tau(h)\mid A=a_2]
 = r(h,a_2)+(1-\epsilon)M.
\end{aligned}
\]
Subtracting gives the stated expression, and the reversal condition follows
immediately.
\end{proof}

\section{Additional Results for Section~\ref{sec:stable_reduction}}
\label{app:sec3}

This appendix provides technical details for the deployable action geometry
defined in Section~\ref{sec:stable_reduction}. Appendix~\ref{app:future_not_identified}
shows why an unrestricted future response law is not identified from observed
periods alone. Appendix~\ref{app:anchored_scales} justifies the anchored link
family and the need to fix a contrast scale. Appendix~\ref{app:normalization}
and Appendix~\ref{app:profiled_calibration} give the normalization convention
and the closed form of the profiled calibration. Appendix~\ref{app:oracle_properties}
records basic properties of the oracle criterion.

\subsection{Unrestricted future deployment is not identified}
\label{app:future_not_identified}

\begin{proposition}[Observed periods do not identify the future optimum]
\label{prop:future_action_not_identified}
Assume the one-step causal conditions hold on periods \(t=1,\ldots,T\), so that
each observed-period response \(m_t\) is identified on \(\mathcal A_\star\).
Then the observed law on periods \(1,\ldots,T\) does not identify the future
response \(m_{T+1}\), and hence does not identify the future-optimal action.
More precisely, for any \(h\in\mathcal H\) and any \(a,a'\in\mathcal A_\star\),
there exist two data-generating processes \(\mathbb P^+\) and \(\mathbb P^-\)
that induce the same observed law on periods \(1,\ldots,T\) but satisfy
\[
    m^{+}_{T+1}(h,a)>m^{+}_{T+1}(h,a'),
    \qquad
    m^{-}_{T+1}(h,a)<m^{-}_{T+1}(h,a').
\]
\end{proposition}

\begin{proof}
Fix any observed-period process \(\mathbb P\) satisfying the one-step causal
assumptions on \(t=1,\ldots,T\). We construct two extensions beyond the observed
periods.

For \(t=1,\ldots,T\), let \(\mathbb P^+\) and \(\mathbb P^-\) coincide with
\(\mathbb P\). Hence the observed data law is identical under the two processes.
For period \(T+1\), set \(H_{T+1}=h\) almost surely under both extensions, and
choose any assignment kernel whose support contains \(\mathcal A_\star\). Define
future response functions \(\widetilde m^+\) and \(\widetilde m^-\) such that
\[
    \widetilde m^+(h,a)>\widetilde m^+(h,a'),
    \qquad
    \widetilde m^-(h,a)<\widetilde m^-(h,a').
\]
For example, take
\(\widetilde m^+(h,a)=1,\widetilde m^+(h,a')=0\) and
\(\widetilde m^-(h,a)=0,\widetilde m^-(h,a')=1\), with arbitrary admissible
values elsewhere.

Finally choose potential outcomes in period \(T+1\) so that
\[
    \mathbb E_{\mathbb P^\pm}
    \!\left[
        Y_{T+2}(b)\mid H_{T+1}=h
    \right]
    =
    \widetilde m^\pm(h,b),
    \qquad b\in\mathcal A_\star,
\]
for instance by setting
\(Y_{T+2}(b)=\widetilde m^\pm(h,b)+\varepsilon_b\) with mean-zero noise.
The two extensions agree on all observed periods but rank \(a\) and \(a'\)
oppositely in period \(T+1\). Thus the future response and the future optimum are
not identified from the observed-period law.
\end{proof}

\subsection{Anchored scales and canonical links}
\label{app:anchored_scales}

\subsubsection{Contrast-scale non-canonicity}
\label{app:no_canonical_contrast}

\begin{proposition}[Anchored factorizations are not canonical without fixing scale]
\label{prop:no_canonical_contrast}
Suppose an anchored contrast admits
\[
    \Delta_t(h,a)=s_t(h)\rho(h,a),
    \qquad s_t(h)\ge0,\qquad \rho(h,a_0)=0.
\]
For any \(p>0\), define
\[
    \Delta_t^{[p]}(h,a)
    :=
    \operatorname{sgn}\{\Delta_t(h,a)\}\,
    |\Delta_t(h,a)|^p .
\]
Then
\[
    \Delta_t^{[p]}(h,a)
    =
    s_t(h)^p \rho^{[p]}(h,a),
    \qquad
    \rho^{[p]}(h,a)
    :=
    \operatorname{sgn}\{\rho(h,a)\}\,|\rho(h,a)|^p,
\]
and \(\rho^{[p]}(h,a_0)=0\).
\end{proposition}

\begin{proof}
By direct substitution,
\[
\Delta_t^{[p]}(h,a)
=
\operatorname{sgn}\{s_t(h)\rho(h,a)\}
|s_t(h)\rho(h,a)|^p
=
s_t(h)^p
\operatorname{sgn}\{\rho(h,a)\}|\rho(h,a)|^p,
\]
where the sign is unchanged because \(s_t(h)\ge0\). The anchoring condition
follows from \(\rho(h,a_0)=0\).
\end{proof}

This construction shows that a stable factorization is not uniquely defined
until the anchored contrast scale is fixed. The link family in the main text
therefore specifies the admissible link family rather than merely a
preprocessing convention.

\subsubsection{Canonical nuisance-eliminating links}
\label{app:canonical_links}

\begin{proof}[Proof of Proposition~\ref{prop:canonical_links}]
We prove the two statements in turn.

For the additive case, fix \(y,y_0\) in the interior of \(I\) and define
\[
    F(b):=\ell(y+b)-\ell(y_0+b)
\]
whenever the arguments lie in \(I\). By assumption \(F\) is constant, so
\[
    0=F'(b)=\ell'(y+b)-\ell'(y_0+b).
\]
Taking \(b=0\) gives \(\ell'(y)=\ell'(y_0)\). Since \(y,y_0\) are arbitrary
interior points, \(\ell'\) is constant on the interior of \(I\). Hence
\(\ell(y)=\alpha y+\beta\) on \(I\), with \(\alpha>0\) because \(\ell\) is
strictly increasing.

For the multiplicative case, let \(J:=\{\log y:y\in I\}\) and define
\(g(x):=\ell(e^x)\). Then \(g\) is continuously differentiable and strictly
increasing. The multiplicative invariance condition gives, whenever the
arguments are in \(J\),
\[
    g(x+c)-g(x_0+c)
    =
    \ell(e^{x+c})-\ell(e^{x_0+c})
    =
    \ell(e^x)-\ell(e^{x_0})
    =
    g(x)-g(x_0).
\]
Thus \(g\) satisfies the additive invariance condition on \(J\). By the first
part, \(g(x)=\alpha x+\beta\) with \(\alpha>0\), and therefore
\[
    \ell(y)=g(\log y)=\alpha\log y+\beta .
\]
\end{proof}

\subsection{Normalization and profiled calibration}
\label{app:normalization_and_profile}

\subsubsection{Normalization convention}
\label{app:normalization}

The factorization
\[
    \Delta_{k,t}(h,a)=s_{k,t}(h)\rho(u(h),a)
\]
is unchanged under the rescaling
\[
    \rho\mapsto c\rho,\qquad s_{k,t}\mapsto s_{k,t}/c,
    \qquad c>0.
\]
We therefore normalize the shape by imposing
\[
    \rho(z,a_0)=0,
    \qquad
    \|\rho(z,\cdot)\|_\varpi^2
    =
    \int_{\mathcal A_\star}\rho(z,a)^2\,\varpi(da)
    =
    1
    \quad \text{for all } z.
\]
To fix the remaining sign ambiguity, choose a nonzero loading
\(b:\mathcal A_\star\to\mathbb R\) satisfying
\[
    b(a_0)=0,
    \qquad
    \|b\|_\varpi=1,
\]
and require
\[
    \int_{\mathcal A_\star}\rho(z,a)b(a)\,\varpi(da)\ge0
    \quad \text{for all } z.
\]
This convention fixes the within-candidate scale and orientation of \(\rho\). The profiled loss in Section~\ref{subsec:oracle_target} additionally normalizes the anchored response energy, so cross-link comparisons are made at the level of geometry rather than contrast amplitude.

\subsubsection{Profiled calibration in closed form}
\label{app:profiled_calibration}

\begin{lemma}[Closed form of profiled calibration]
\label{lem:profiled_calibration}
Suppose \(\rho\) satisfies the normalization in Appendix~\ref{app:normalization}.
For fixed \(k,t,h,u,\rho\), define
\[
    \Delta(a):=\Delta_{k,t}(h,a),
    \qquad
    R(a):=\rho(u(h),a).
\]
If \(\|\Delta\|_\varpi>0\), then the minimizer of
\[
    \min_{s\ge0}\|\Delta-sR\|_\varpi^2
\]
is
\[
    s^\star_{k,t}(h;u,\rho)
    =
    \left[
        \langle \Delta, R\rangle_\varpi
    \right]_+
    =
    \left[
        \int_{\mathcal A_\star}
        \Delta_{k,t}(h,a)\rho(u(h),a)\,\varpi(da)
    \right]_+ .
\]
The same \(s^\star\) minimizes the normalized objective
\[
    \min_{s\ge0}
    \frac{\|\Delta-sR\|_\varpi^2}{\|\Delta\|_\varpi^2}.
\]
If \(\|\Delta\|_\varpi=0\), the normalized contribution is defined as \(0\), and
\(s=0\) is admissible.
\end{lemma}

\begin{proof}
By the normalization \(\|R\|_\varpi=1\),
\[
    \|\Delta-sR\|_\varpi^2
    =
    \|\Delta\|_\varpi^2
    -2s\langle\Delta,R\rangle_\varpi
    +s^2 .
\]
This is a convex quadratic over \(s\ge0\), whose minimizer is
\(s^\star=[\langle\Delta,R\rangle_\varpi]_+\). Dividing by the positive constant
\(\|\Delta\|_\varpi^2\) does not change the minimizer. If
\(\|\Delta\|_\varpi=0\), then \(\Delta=0\) in \(L_2(\varpi)\), and the stated
zero-contribution convention applies.
\end{proof}

\subsection{Properties of the oracle criterion}
\label{app:oracle_properties}

\subsubsection{Scale invariance of the profiled loss}
\label{app:scale_invariance}

\begin{lemma}[Scale invariance]
\label{lem:scale_invariance}
Let \(c_t(h)>0\) be measurable and define
\[
    \widetilde\Delta_{k,t}(h,a):=c_t(h)\Delta_{k,t}(h,a).
\]
For any candidate \((u,\rho)\), the pointwise normalized profiled loss is
unchanged when \(\Delta_{k,t}\) is replaced by \(\widetilde\Delta_{k,t}\).
\end{lemma}

\begin{proof}
If \(\|\Delta_{k,t}(h,\cdot)\|_\varpi=0\), then also
\(\|\widetilde\Delta_{k,t}(h,\cdot)\|_\varpi=0\), and both contributions are
defined as \(0\). Otherwise,
\[
\begin{aligned}
\inf_{s\ge0}
\frac{
\|\widetilde\Delta_{k,t}(h,\cdot)-s\rho(u(h),\cdot)\|_\varpi^2
}{
\|\widetilde\Delta_{k,t}(h,\cdot)\|_\varpi^2
}
&=
\inf_{s\ge0}
\frac{
\|c_t(h)\Delta_{k,t}(h,\cdot)-s\rho(u(h),\cdot)\|_\varpi^2
}{
c_t(h)^2\|\Delta_{k,t}(h,\cdot)\|_\varpi^2
} \\
&=
\inf_{r\ge0}
\frac{
\|\Delta_{k,t}(h,\cdot)-r\rho(u(h),\cdot)\|_\varpi^2
}{
\|\Delta_{k,t}(h,\cdot)\|_\varpi^2
},
\end{aligned}
\]
where \(r=s/c_t(h)\). This proves the claim.
\end{proof}

\subsubsection{Exact specification implies zero approximation loss}
\label{app:exact_spec_zero_loss}

\begin{lemma}[Exact calibrated shape gives zero loss]
\label{lem:exact_spec_zero_loss}
Suppose that for some \(k\in\mathcal K\), \(u\in\mathfrak U\), and
\(\rho\in\mathfrak R_k(u)\),
\[
    \Delta_{k,t}(h,a)=s_{k,t}(h)\rho(u(h),a)
\]
for every \(t=1,\ldots,T\), almost every \(h\), and every
\(a\in\mathcal A_\star\), with \(s_{k,t}(h)\ge0\). Then
\[
    \mathcal A_k(u,\rho)=0.
\]
\end{lemma}

\begin{proof}
Fix \(t,h\). If
\(\|\Delta_{k,t}(h,\cdot)\|_\varpi=0\), the pointwise contribution to
\(\mathcal A_k(u,\rho)\) is \(0\) by convention. Otherwise, choosing
\(s=s_{k,t}(h)\) makes the numerator of the profiled loss equal to \(0\). Hence
the contribution is \(0\) pointwise, and the expectation over \(t,h\) is also
\(0\).
\end{proof}

\subsubsection{Non-uniqueness of unrestricted representations}
\label{app:representation_nonuniqueness}

\begin{proposition}[Unrestricted representation search is non-unique]
\label{prop:representation_nonuniqueness}
Suppose \((u,\rho)\) satisfies
\[
    \Delta_{k,t}(h,a)=s_{k,t}(h)\rho(u(h),a)
\]
for all \(t,h,a\). Let \(w:\mathcal H\to\mathbb R^r\) be any measurable map and
define
\[
    \widetilde u(h):=(u(h),w(h)),
    \qquad
    \widetilde\rho((z,r),a):=\rho(z,a).
\]
Then the same factorization holds with \((\widetilde u,\widetilde\rho)\), and
the induced action geometry is unchanged:
\[
    \arg\max_{a\in\mathcal A_\star}
    \widetilde\rho(\widetilde u(h),a)
    =
    \arg\max_{a\in\mathcal A_\star}
    \rho(u(h),a).
\]
\end{proposition}

\begin{proof}
By construction,
\[
    \widetilde\rho(\widetilde u(h),a)
    =
    \widetilde\rho((u(h),w(h)),a)
    =
    \rho(u(h),a).
\]
Thus the factorization and the induced action correspondence are unchanged.
This shows that the representation itself is not identifiable without controlled
classes or additional restrictions; the decision-relevant object is the induced
geometry.
\end{proof}

\subsection{Proof of Theorem~\ref{thm:approximate_deployability_main}}
\label{app:proof_approximate_deployability}

\begin{proof}
Fix a deployment period \(t\in\mathcal T_{\mathrm{dep}}\) and a history \(h\)
for which Assumption~\ref{ass:approximate_persistence_main} holds. Let
\(a^\rho(h)\in G^\star(h)\), so that
\[
    \rho^\star(u^\star(h),a^\rho(h))
    \ge
    \rho^\star(u^\star(h),a)
    \qquad
    \forall a\in\mathcal A_\star .
\]
Let \(a_t^{\mathrm{opt}}(h)\in\arg\max_{a\in\mathcal A_\star}m_t(h,a)\).
Because \(\ell_{k^\star}\) is strictly increasing and anchoring subtracts the
action-independent constant \(\ell_{k^\star}(m_t(h,a_0))\), the maximizers of
\(m_t(h,\cdot)\) and \(\Delta_{k^\star,t}(h,\cdot)\) coincide. Hence
\[
    \Delta_{k^\star,t}(h,a_t^{\mathrm{opt}}(h))
    \ge
    \Delta_{k^\star,t}(h,a^\rho(h)),
\]
which gives the lower bound.

For the upper bound, write
\[
    \Delta_{k^\star,t}(h,a)
    =
    s_t(h)\rho^\star(u^\star(h),a)+\xi_t(h,a),
    \qquad s_t(h)>0 .
\]
Then
\[
\begin{aligned}
&\Delta_{k^\star,t}(h,a_t^{\mathrm{opt}}(h))
-
\Delta_{k^\star,t}(h,a^\rho(h)) \\
&=
s_t(h)
\Big[
\rho^\star(u^\star(h),a_t^{\mathrm{opt}}(h))
-
\rho^\star(u^\star(h),a^\rho(h))
\Big] \\
&\qquad
+
\xi_t(h,a_t^{\mathrm{opt}}(h))
-
\xi_t(h,a^\rho(h)).
\end{aligned}
\]
The bracketed term is non-positive because \(a^\rho(h)\) maximizes
\(\rho^\star(u^\star(h),\cdot)\) and \(s_t(h)>0\). Therefore,
\[
\begin{aligned}
&\Delta_{k^\star,t}(h,a_t^{\mathrm{opt}}(h))
-
\Delta_{k^\star,t}(h,a^\rho(h)) \\
&\le
\xi_t(h,a_t^{\mathrm{opt}}(h))
-
\xi_t(h,a^\rho(h))
\le
2\sup_{a\in\mathcal A_\star}|\xi_t(h,a)|.
\end{aligned}
\]
This proves the stated bound. If \(\xi_t\equiv0\), the upper bound is zero, so
every selector from \(G^\star\) is optimal in period \(t\). Since \(t\) and \(h\)
were arbitrary on the set where the persistence condition holds, the result
follows.
\end{proof}

\section{Additional Results for Section~\ref{sec:orthogonal_estimation}}
\label{app:sec4}

This appendix gives the technical details behind the estimator in
Section~\ref{sec:orthogonal_estimation}. Appendix~\ref{app:sec4_stage1} states
the Stage-I conditions. Appendix~\ref{app:sec4_scores} proves identification and
Neyman orthogonality of the localized score. Appendix~\ref{app:sec4_pilot_rates}
derives the finite-grid pilot bound used in~\eqref{eq:stage1_grid_bound}.
Appendix~\ref{app:sec4_profiled} records stability of the profiled finite-grid
criterion. Appendix~\ref{app:sec4_stage2} proves the finite-library oracle
inequality. Appendix~\ref{app:sec4_scope} clarifies the scope of the guarantee,
and Appendix~\ref{app:sec4_policy} gives the criterion-to-policy bridge used in
the decision consequence.

\subsection{Stage-I regularity}
\label{app:sec4_stage1}

\begin{assumption}[Stage-I regularity]
\label{ass:sec4_stage1}
For each observed period \(t\in\{1,\ldots,T\}\), the following conditions hold on
the target region.

\begin{enumerate}
    \item \textbf{One-step identification.}
    Consistency, no interference, overlap, and one-step exchangeability hold as
    in Assumption~\ref{ass:one_step_identifiability}. Hence
    \(\mu_{0,t}(h,a):=\mathbb E[Y\mid H=h,A=a,\tau=t]=m_t(h,a)\) on the
    target support.

    \item \textbf{Continuous-action overlap.}
    The generalized propensity density
    \(g_{0,t}(a\mid h):=f_{A\mid H,\tau=t}(a\mid h)\) exists and satisfies
    \(0<\underline g\le g_{0,t}(a\mid h)\le \overline g<\infty\) on the
    target region.

    \item \textbf{Smoothness and boundedness.}
    The response \(m_t(h,a)\) is uniformly bounded and \(\alpha_a\)-H\"older in
    \(a\). For each \(u\in\mathfrak U_n\), the law of \(U=u(H)\mid\tau=t\)
    admits a density bounded and bounded away from zero on its target region,
    and the representation-indexed response is \(\alpha_u\)-H\"older in \(u\).

    \item \textbf{Representation localization.}
    The finite-grid pilot statement is written for \(m_t(h,a)\). This holds if,
    on the Stage-II evaluation region, \(m_t(h,a)\) is identified by localization
    in \(u(h)\) for each candidate \(u\in\mathfrak U_n\). If a candidate
    representation is not sufficient, the same analysis applies to the projected response \(m_{t,u}(z,a):=\mathbb E[m_t(H,a)\mid u(H)=z,\tau=t]\), and the projection error is absorbed into the restricted oracle approximation term.

    \item \textbf{Kernels and bandwidths.}
    The kernels \(K\) and \(L\) are bounded, compactly supported, integrate to
    one, and have finite second moments. The bandwidths satisfy
    \(\lambda\to0\), \(b\to0\), \(n_t\lambda^{d_u}b\to\infty\), and
    \(n_t\lambda^{d_u}\to\infty\).

    \item \textbf{Cross-fitted nuisance rates.}
    The nuisance estimators are fitted on folds independent of the observations
    on which their scores are evaluated and satisfy
    \(\|\hat\mu_t-\mu_{0,t}\|_{L_2(P_t)}=O_p(r_{\mu,t,n})\) and
    \(\|\hat g_t-g_{0,t}\|_{L_2(P_t)}=O_p(r_{g,t,n})\), with
    \(\bar r_n:=\max_t r_{\mu,t,n}r_{g,t,n}=o(1)\).

    \item \textbf{Log-link positivity.}
    For the log link, \(m_t(h,a)\) is uniformly bounded away from zero on
    the link-valid target region.
\end{enumerate}
\end{assumption}

\subsection{Localized score: identification and orthogonality}
\label{app:sec4_scores}

Fix \(t,u,h,a\), and write \(U=u(H)\). Define
\[
\nu_{t,\lambda,b}(u,h,a)
:=
\mathbb E\!\left[
\frac{\mathbf 1\{\tau=t\}}{\pi_t}
L_\lambda(U-u(h))
\int_{\mathcal A_\star}K_b(s-a)m_t(H,s)\,ds
\right],
\]
and
\[
\delta_{t,\lambda}(u,h)
:=
\mathbb E\!\left[
\frac{\mathbf 1\{\tau=t\}}{\pi_t}L_\lambda(U-u(h))
\right].
\]
Let \(\bar m_{t,\lambda,b}(u,h,a):=
\nu_{t,\lambda,b}(u,h,a)/\delta_{t,\lambda}(u,h)\).

\begin{proposition}[Score identification and Neyman orthogonality]
\label{prop:sec4_score}
Under Assumption~\ref{ass:sec4_stage1},
\[
\mathbb E[\phi_{t,\lambda,b}(O;u,h,a,\eta_{0,t})]
=
\nu_{t,\lambda,b}(u,h,a),
\qquad
\mathbb E[\psi_{t,\lambda}(O;u,h)]
=
\delta_{t,\lambda}(u,h).
\]
Moreover, the map
\(\eta_t\mapsto
\mathbb E[\phi_{t,\lambda,b}(O;u,h,a,\eta_t)]\) is Gateaux-orthogonal at
\(\eta_{0,t}=(\mu_{0,t},g_{0,t})\).
\end{proposition}

\begin{proof}
At the true nuisances,
\(\mathbb E[Y-\mu_{0,t}(H,A)\mid H,A,\tau=t]=0\), so the residualized term in
\(\phi\) has mean zero. The remaining term equals
\[
\mathbb E\!\left[
\frac{\mathbf 1\{\tau=t\}}{\pi_t}
L_\lambda(U-u(h))
\int_{\mathcal A_\star}K_b(s-a)\mu_{0,t}(H,s)\,ds
\right],
\]
which is \(\nu_{t,\lambda,b}(u,h,a)\) because
\(\mu_{0,t}(H,s)=m_t(H,s)\) under one-step identification. The identity for
\(\psi\) is immediate.

For orthogonality, let \(\eta_{t,r}=(\mu_{t,r},g_{t,r})\) be a regular path
through \(\eta_{0,t}\) with derivatives \((\dot\mu_t,\dot g_t)\). Differentiating
the score mean at \(r=0\), the derivative in the \(\mu\)-direction is
\[
\mathbb E\!\left[
\frac{\mathbf 1\{\tau=t\}}{\pi_t}L_\lambda(U-u(h))
\left\{
-\frac{K_b(A-a)}{g_{0,t}(A\mid H)}\dot\mu_t(H,A)
+
\int K_b(s-a)\dot\mu_t(H,s)\,ds
\right\}
\right],
\]
which is zero after conditioning on \((H,\tau=t)\). The derivative in the
\(g\)-direction is proportional to
\[
\mathbb E\!\left[
\frac{\mathbf 1\{\tau=t\}}{\pi_t}L_\lambda(U-u(h))
\frac{K_b(A-a)\dot g_t(A\mid H)}{g_{0,t}(A\mid H)^2}
\{Y-\mu_{0,t}(H,A)\}
\right],
\]
which is zero after conditioning on \((H,A,\tau=t)\). Hence the first derivative
vanishes in every regular nuisance direction.
\end{proof}

\subsection{Pilot expansion and finite-grid control}
\label{app:sec4_pilot_rates}

\begin{lemma}[Ratio and log-ratio stability]
\label{lem:sec4_ratio}
If \(x_n-x=O_p(r_n)\), \(y_n-y=O_p(r_n)\), and \(y\) is bounded away from zero,
then
\[
\frac{x_n}{y_n}-\frac{x}{y}
=
\frac{x_n-x}{y}
-
\frac{x}{y^2}(y_n-y)
+
O_p(r_n^2).
\]
If also \(x\) is bounded away from zero, then
\[
(\log x_n-\log y_n)-(\log x-\log y)
=
\frac{x_n-x}{x}
-
\frac{y_n-y}{y}
+
O_p(r_n^2).
\]
\end{lemma}

\begin{proof}
Both claims are first-order Taylor expansions on compact subsets bounded away
from zero.
\end{proof}

Let \(\bar\Delta_{k,t,\lambda,b}\) denote the anchored contrast constructed from
\(\bar m_{t,\lambda,b}\) with the same link \(\ell_k\).

\begin{lemma}[Pointwise anchored-pilot expansion]
\label{lem:sec4_pointwise_pilot}
Under Assumption~\ref{ass:sec4_stage1}, for every fixed \((k,t,u,h,a)\),
\[
\hat\Delta_{k,t}(u,h,a)-\Delta_{k,t}(h,a)
=
O_p\!\left(
(n_t\lambda^{d_u}b)^{-1/2}
+
\bar r_n
+
\lambda^{\alpha_u}
+
b^{\alpha_a}
\right).
\]
\end{lemma}

\begin{proof}
By Proposition~\ref{prop:sec4_score} and cross-fitting,
\[
\hat\nu_{t,\lambda,b}(u,h,a)-\nu_{t,\lambda,b}(u,h,a)
=
(\mathbb P_n-P)\phi_{t,\lambda,b}(O;u,h,a,\eta_{0,t})
+
O_p(\bar r_n),
\]
where the remainder is second order in the nuisance errors. The empirical-process
term has order \(O_p((n_t\lambda^{d_u}b)^{-1/2})\). Similarly,
\(\hat\delta_{t,\lambda}(u,h)-\delta_{t,\lambda}(u,h)
=O_p((n_t\lambda^{d_u})^{-1/2})\), which is no larger than the displayed
variance order. Lemma~\ref{lem:sec4_ratio} transfers these rates to the ratio
estimator and, under log-link positivity, to the log-ratio estimator. The bias
\(\bar m_{t,\lambda,b}(u,h,a)-m_t(h,a)\) is the standard kernel smoothing bias in
the representation and action coordinates, namely
\(O(\lambda^{\alpha_u}+b^{\alpha_a})\). Anchoring at \(a_0\) preserves the same
order.
\end{proof}

\begin{proposition}[Finite-grid pilot control]
\label{prop:sec4_grid_pilot}
Under Assumption~\ref{ass:sec4_stage1}, there exists a deterministic sequence
\(\varepsilon_n^{\mathrm{pilot}}\downarrow0\) such that the finite-grid pilot
bound in~\eqref{eq:stage1_grid_bound} holds.
\end{proposition}

\begin{proof}
Condition on the independent selection sample. Lemma~\ref{lem:sec4_pointwise_pilot}
gives the stated pointwise rate at each evaluation tuple
\((k,u,t,H_i,a^{(m)})\). Since \(\mathcal K\), \(T\), \(\mathfrak U_n\), and the
Stage-II action grid are finite, a union bound over these indices and over
\(i\in\mathcal I^{\mathrm{sel}}\) yields
\[
\sup_{k,u,t,i,m}
\left|
\hat\Delta_{k,t}(u,H_i,a^{(m)})
-
\Delta_{k,t}(H_i,a^{(m)})
\right|
=
O_p(\varepsilon_n^{\mathrm{pilot}}),
\]
where \(\varepsilon_n^{\mathrm{pilot}}\) absorbs the pointwise rate and the
finite-library/grid union factor. This is exactly~\eqref{eq:stage1_grid_bound}.
\end{proof}

\subsection{Finite-grid profiled criterion}
\label{app:sec4_profiled}

Let
\[
\Theta_n
:=
\{(k,u,\rho): k\in\mathcal K,\ u\in\mathfrak U_n,\ 
\rho\in\mathfrak R_{k,n}(u)\}.
\]
For \(\theta=(k,u,\rho)\) and \(i\in\mathcal I^{\mathrm{sel}}\), define the
population grid energy
\(V_i(\theta)=\sum_m w_m\Delta_{k,\tau_i}(H_i,a^{(m)})^2\).

\begin{assumption}[Finite-grid stability]
\label{ass:sec4_grid_stability}
On the nonzero-energy selection region, there exists \(\underline v>0\) such
that \(\inf_{\theta\in\Theta_n}\min_i V_i(\theta)\ge \underline v\). The shape
classes are uniformly bounded and normalized, and the quadrature rule satisfies
\[
\sup_{\theta\in\Theta_n,i}
\left|
\sum_m w_m \rho(u(H_i),a^{(m)})^2 - 1
\right|
\le q_M,\qquad q_M\downarrow0.
\]
Population zero-energy terms are handled by the zero-energy convention in the
main text and may be removed from the normalized criterion.
\end{assumption}

\begin{lemma}[Profiled calibration on the finite grid]
\label{lem:sec4_profiled}
For fixed \((\theta,i)\), write
\(\Delta_{i,m}=\Delta_{k,\tau_i}(H_i,a^{(m)})\) and
\(\rho_{i,m}=\rho(u(H_i),a^{(m)})\). The gridwise minimizer of
\(\min_{s\ge0}\sum_m w_m(\Delta_{i,m}-s\rho_{i,m})^2\) is
\[
s_i^\star
=
\left[
\frac{\sum_m w_m\Delta_{i,m}\rho_{i,m}}
{\sum_m w_m\rho_{i,m}^2}
\right]_+ .
\]
The same \(s_i^\star\) minimizes the normalized pointwise loss. Under
Assumption~\ref{ass:sec4_grid_stability}, replacing population contrasts by
pilots changes the normalized profiled loss by
\(O_p(\varepsilon_n^{\mathrm{pilot}})\) uniformly over \(\Theta_n\).
\end{lemma}

\begin{proof}
The unnormalized grid objective is the quadratic
\[
Q_i(s)=
\sum_m w_m\Delta_{i,m}^2
-
2s\sum_m w_m\Delta_{i,m}\rho_{i,m}
+
s^2\sum_m w_m\rho_{i,m}^2,
\]
so the constrained minimizer is the displayed positive part. Dividing by the
energy \(\sum_m w_m\Delta_{i,m}^2\) does not change the minimizer. Uniform pilot
stability follows because the numerator, denominator, and profiled quadratic are
locally Lipschitz in the finite vector
\((\Delta_{i,1},\ldots,\Delta_{i,M})\) on the event where the energy is bounded
below by \(\underline v/2\). Proposition~\ref{prop:sec4_grid_pilot} and
Assumption~\ref{ass:sec4_grid_stability} give this event with probability
tending to one.
\end{proof}

\subsection{Finite-library oracle inequality}
\label{app:sec4_stage2}

\begin{lemma}[Uniform empirical approximation]
\label{lem:sec4_uniform_empirical}
Assume Proposition~\ref{prop:sec4_grid_pilot},
Assumption~\ref{ass:sec4_grid_stability}, and that each
\(\mathfrak R_{k,n}(u)\) admits a bounded finite-dimensional Lipschitz
parameterization of dimension at most \(p_n\). Then
\[
\sup_{\theta\in\Theta_n}
\left|
\hat{\mathcal L}^{\mathrm{sel}}(\theta)-\mathcal L(\theta)
\right|
=
O_p\!\left(
\varepsilon_n^{\mathrm{pilot}}
+
\sqrt{\frac{\log|\mathfrak U_n|+p_n\log n}{n_{\mathrm{sel}}}}
+
q_M
\right).
\]
\end{lemma}

\begin{proof}
The difference decomposes into pilot error, sampling error, and quadrature error.
By Lemma~\ref{lem:sec4_profiled}, replacing population contrasts by pilot
contrasts changes the profiled normalized loss by
\(O_p(\varepsilon_n^{\mathrm{pilot}})\) uniformly. Conditional on the pilots being
replaced by population contrasts, the criterion is a bounded empirical process
over a finite representation library and \(p_n\)-dimensional Lipschitz shape
sieves, giving
\[
O_p\!\left(
\sqrt{\frac{\log|\mathfrak U_n|+p_n\log n}{n_{\mathrm{sel}}}}
\right).
\]
The finite-grid approximation contributes \(q_M\). The deterministic penalties
are identical in empirical and population criteria, so they do not affect the
supremum difference.
\end{proof}

\begin{proof}[Proof of Theorem~\ref{thm:main_learnability}]
Let \(\hat\theta=(\hat k,\hat u,\hat\rho)\) and let
\(\theta_n^\star\in\arg\min_{\theta\in\Theta_n}\mathcal L(\theta)\). Empirical
optimality gives
\(\hat{\mathcal L}^{\mathrm{sel}}(\hat\theta)
\le
\hat{\mathcal L}^{\mathrm{sel}}(\theta_n^\star)\). Hence
\[
\mathcal L(\hat\theta)-\mathcal L(\theta_n^\star)
\le
2\sup_{\theta\in\Theta_n}
|\hat{\mathcal L}^{\mathrm{sel}}(\theta)-\mathcal L(\theta)|.
\]
Lemma~\ref{lem:sec4_uniform_empirical} yields the claimed \(O_p(r_n)\) bound.
\end{proof}

\begin{corollary}[Link selection consistency]
\label{cor:main_morphology_consistency}
Let \(V_n(k):=\inf_{u,\rho}\mathcal L_k(u,\rho)\), where
\(u\in\mathfrak U_n\) and \(\rho\in\mathfrak R_{k,n}(u)\). Suppose \(k^\star\)
is the unique minimizer and
\(\gamma_n:=\min_{k\neq k^\star}\{V_n(k)-V_n(k^\star)\}\). If
\(r_n=o(\gamma_n)\), then \(\mathbb P(\hat k=k^\star)\to1\).
\end{corollary}

\begin{proof}
If \(\hat k\neq k^\star\), then
\(\mathcal L(\hat k,\hat u,\hat\rho)-\inf_{\theta\in\Theta_n}\mathcal L(\theta)
\ge \gamma_n\). Theorem~\ref{thm:main_learnability} bounds the left-hand side by
\(O_p(r_n)\). Therefore \(r_n=o(\gamma_n)\) implies
\(\mathbb P(\hat k\neq k^\star)\to0\).
\end{proof}

\subsection{Scope of the finite-library guarantee}
\label{app:sec4_scope}

Theorem~\ref{thm:main_learnability} is a finite-library, criterion-level result.
It should be read as follows. The population object is the restricted oracle criterion from Section~\ref{sec:stable_reduction}. The statistical statement controls the excess of this criterion over the admissible library \(\Theta_n\). It does not claim identification of a unique latent representation, a globally stable identity response surface, or convergence of an unrestricted nonconvex neural representation learner. In implementations with flexible encoders, the theorem applies to the selected finite candidate library
or to a sieve approximation of that library; additional optimization error can
be added to \(r_n\) in the usual way.

\subsection{Decision consequences}
\label{app:sec4_policy}

The oracle inequality controls a geometry criterion. We next state sufficient
conditions under which this criterion control yields policy regret. The result is
stated for the finite decision grid used by the estimator. Continuous-action
extensions require the corresponding continuous-action margin or curvature
condition.

Let \(\mathcal A_{\mathrm{dec}}\subset\mathcal A_\star\) be the finite decision
grid. Let \(\mathcal L_n^\star:=\inf_{\theta\in\Theta_n}\mathcal L(\theta)\) and
\(\Theta_n^\star:=\arg\min_{\theta\in\Theta_n}\mathcal L(\theta)\). For
\(\theta=(k,u,\rho)\) and \(\theta'=(k',u',\rho')\), define
\[
d_\infty(\theta,\theta')
:=
\operatorname*{ess\,sup}_{h}
\max_{a\in\mathcal A_{\mathrm{dec}}}
|\rho(u(h),a)-\rho'(u'(h),a)|,
\quad
d_\infty(\theta,\Theta_n^\star)
:=
\inf_{\theta'\in\Theta_n^\star}d_\infty(\theta,\theta').
\]

\begin{assumption}[Criterion-to-geometry modulus]
\label{ass:criterion_to_geometry}
There exists a nondecreasing deterministic modulus \(\Psi_n\), with
\(\Psi_n(\delta)\to0\) as \(\delta\downarrow0\), such that
\[
d_\infty(\theta,\Theta_n^\star)
\le
\Psi_n\{\mathcal L(\theta)-\mathcal L_n^\star\}
\qquad
\forall \theta\in\Theta_n .
\]
\end{assumption}

\begin{proposition}[Geometry recovery from criterion learnability]
\label{prop:geometry_from_oracle}
Under Theorem~\ref{thm:main_learnability} and
Assumption~\ref{ass:criterion_to_geometry},
\[
d_\infty\{(\hat k,\hat u,\hat\rho),\Theta_n^\star\}
=
O_p\{\Psi_n(r_n)\}.
\]
\end{proposition}

\begin{proof}
Let \(\hat\theta=(\hat k,\hat u,\hat\rho)\). By
Theorem~\ref{thm:main_learnability},
\(\mathcal L(\hat\theta)-\mathcal L_n^\star=O_p(r_n)\). Applying
Assumption~\ref{ass:criterion_to_geometry} gives the result.
\end{proof}

Fix an oracle representative
\(\theta_n^\star=(k_n^\star,u_n^\star,\rho_n^\star)\in\Theta_n^\star\), and let
\(a_n^\star(h)\in\arg\max_{a\in\mathcal A_{\mathrm{dec}}}
\rho_n^\star(u_n^\star(h),a)\).

\begin{assumption}[Deployment bridge and margin]
\label{ass:sec4_policy_bridge}
For every deployment period \(t\in\mathcal T_{\mathrm{dep}}\), assume
\[
\Delta_{k_n^\star,t}(h,a)
=
s_t(h)\rho_n^\star(u_n^\star(h),a)+\xi_t(h,a),
\qquad
0<s_t(h)\le \overline s ,
\]
for all \(a\in\mathcal A_{\mathrm{dec}}\). Also assume a linked-to-raw comparison:
for some \(C_{\mathrm{lr}}<\infty\),
\[
m_t(h,a_t^{\mathrm{opt}}(h))-m_t(h,a)
\le
C_{\mathrm{lr}}
\{\Delta_{k_n^\star,t}(h,a_t^{\mathrm{opt}}(h))
-\Delta_{k_n^\star,t}(h,a)\}.
\]
Finally, define the grid margin
\[
\Gamma_n(h)
:=
\rho_n^\star(u_n^\star(h),a_n^\star(h))
-
\max_{a\in\mathcal A_{\mathrm{dec}},\,a\neq a_n^\star(h)}
\rho_n^\star(u_n^\star(h),a),
\]
and assume
\(\sup_{t\in\mathcal T_{\mathrm{dep}}}
\mathbb P\{\Gamma_n(H_t)\le\epsilon\mid\tau=t\}\le C\epsilon^\kappa\).
\end{assumption}

For the identity link, one may take \(C_{\mathrm{lr}}=1\). For the log link, the
linked-to-raw comparison follows from a uniform upper bound on \(m_t\) over the
deployment region.

\begin{corollary}[Deployment regret from finite-library geometry error]
\label{cor:main_policy_regret}
Suppose Proposition~\ref{prop:geometry_from_oracle} and
Assumption~\ref{ass:sec4_policy_bridge} hold. Define
\[
\mathrm{Pers}_t
:=
\mathbb E_t\!\left[
\max_{a\in\mathcal A_{\mathrm{dec}}}|\xi_t(H_t,a)|
\right].
\]
Then, for each \(t\in\mathcal T_{\mathrm{dep}}\),
\[
\mathrm{Regret}_t(\hat a)
:=
\mathbb E_t[
m_t(H_t,a_t^{\mathrm{opt}}(H_t))
-
m_t(H_t,\hat a(H_t))
]
=
O_p\!\left(\Psi_n(r_n)^{\kappa+1}\right)
+
O(\mathrm{Pers}_t).
\]
\end{corollary}

\begin{proof}
Let
\[
\delta_n
:=
\operatorname*{ess\,sup}_{h}
\max_{a\in\mathcal A_{\mathrm{dec}}}
|\hat\rho(\hat u(h),a)-\rho_n^\star(u_n^\star(h),a)|.
\]
By Proposition~\ref{prop:geometry_from_oracle},
\(\delta_n=O_p\{\Psi_n(r_n)\}\). On the event \(\{\delta_n\le\delta\}\), if
\(\hat a(h)\neq a_n^\star(h)\), then \(\Gamma_n(h)\le2\delta\); otherwise the
uniform perturbation could not change the maximizer. Hence
\[
0\le
\rho_n^\star(u_n^\star(h),a_n^\star(h))
-
\rho_n^\star(u_n^\star(h),\hat a(h))
\le
2\delta\,\mathbf 1\{\Gamma_n(h)\le2\delta\}.
\]
Using the deployment bridge and \(s_t(h)\le\overline s\),
\[
\Delta_{k_n^\star,t}(h,a_t^{\mathrm{opt}}(h))
-
\Delta_{k_n^\star,t}(h,\hat a(h))
\le
2\overline s\,\delta\,\mathbf 1\{\Gamma_n(h)\le2\delta\}
+
2\max_{a\in\mathcal A_{\mathrm{dec}}}|\xi_t(h,a)|.
\]
Taking conditional expectation and applying the margin condition gives
\(O(\delta^{\kappa+1})+O(\mathrm{Pers}_t)\) for linked-response regret. The
linked-to-raw comparison transfers the bound to raw regret. Substituting
\(\delta=\delta_n\) completes the proof.
\end{proof}

\newpage
\section{Implementation Algorithm}
\label{app:Implement}

\begin{algorithm}[htp]
\caption{Selection and optional honest refit for deployable action geometry}
\label{alg:compact_two_stage}
\begin{algorithmic}[1]
\REQUIRE Data $\mathcal D=\{(T_i,H_i,A_i,Y_i)\}_{i=1}^n$, link set $\mathcal K$, representation library $\mathfrak U_n$, normalized shape classes $\{\mathfrak R_{k,n}(u)\}$, kernels $(K_b,L_\lambda)$, grid $\{(a^{(m)},w_m)\}_{m=1}^M$, penalties, truncation $\tau_n$.
\ENSURE $(\hat k,\hat u,\hat\rho)$ and rule $\hat a(h)\in\arg\max_{a\in\mathcal A_\star}\hat\rho(\hat u(h),a)$.

\STATE Split $\mathcal D$ into $\mathcal I^{\mathrm{pil}}$, $\mathcal I^{\mathrm{sel}}$ (and optionally $\mathcal I^{\mathrm{ref}}$), preserving temporal proportions.
\STATE Fit cross-fitted nuisance estimates $\hat\mu_t,\hat g_t$ on $\mathcal I^{\mathrm{pil}}$.

\FOR{each $u\in\mathfrak U_n$ and $i\in\mathcal I^{\mathrm{sel}}$}
    \STATE Compute $\hat m_{t,\lambda,b}(u,H_i,a)=\hat\nu_{t,\lambda,b}(u,H_i,a)/(\hat\delta_{t,\lambda}(u,H_i)\vee\tau_n)$ on the grid.
    \STATE Define anchored pilots:
    \[
    \hat\Delta_{0,T_i,i}^{(m)}(u)=\hat m(u,H_i,a^{(m)})-\hat m(u,H_i,a_0),
    \]
    \[
    \hat\Delta_{1,T_i,i}^{(m)}(u)=\log(\hat m(u,H_i,a^{(m)})\vee\tau_n)-\log(\hat m(u,H_i,a_0)\vee\tau_n).
    \]
    \STATE For each $k\in\mathcal K$, compute
    \[
    \hat V_{k,i}(u)=\sum_{m=1}^M w_m\bigl(\hat\Delta_{k,T_i,i}^{(m)}(u)\bigr)^2.
    \]
\ENDFOR

\FOR{each $(k,u,\rho)$ with $k\in\mathcal K$, $u\in\mathfrak U_n$, $\rho\in\mathfrak R_{k,n}(u)$}
    \FOR{each $i\in\mathcal I^{\mathrm{sel}}$}
        \STATE Compute
        \[
        \hat s_{k,i}(u,\rho)=\Big[\sum_{m=1}^M w_m \hat\Delta_{k,T_i,i}^{(m)}(u)\rho(u(H_i),a^{(m)})\Big]_+.
        \]
        \STATE Define
        \[
        \hat R_{k,i}(u,\rho)=
        \begin{cases}
        \dfrac{\sum_{m=1}^M w_m\bigl(\hat\Delta_{k,T_i,i}^{(m)}(u)-\hat s_{k,i}(u,\rho)\rho(u(H_i),a^{(m)})\bigr)^2}{\hat V_{k,i}(u)},
        & \hat V_{k,i}(u)>0,\\[1.2ex]
        0, & \hat V_{k,i}(u)=0.
        \end{cases}
        \]
    \ENDFOR
    \STATE Form
    \[
    \hat{\mathcal L}^{\mathrm{sel}}_k(u,\rho)
    =|\mathcal I^{\mathrm{sel}}|^{-1}\sum_{i\in\mathcal I^{\mathrm{sel}}}\hat R_{k,i}(u,\rho)
    +\lambda_{u,n}\Omega_u(u)+\lambda_{\rho,n}\Omega_{k,\rho}(\rho).
    \]
\ENDFOR

\STATE Select
\[
(\hat k,\hat u,\hat\rho^{\mathrm{sel}})
\in
\arg\min_{k,u,\rho}\hat{\mathcal L}^{\mathrm{sel}}_k(u,\rho).
\]

\IF{no honest refit is used}
    \STATE Set $\hat\rho\leftarrow \hat\rho^{\mathrm{sel}}$.
\ELSE
    \STATE Rebuild pilots on $\mathcal I^{\mathrm{ref}}$ only for $(\hat k,\hat u)$ and minimize the corresponding refit criterion over $\rho\in\mathfrak R_{\hat k,n}(\hat u)$.
\ENDIF

\STATE Return $\hat a(h)\in\arg\max_{a\in\mathcal A_\star}\hat\rho(\hat u(h),a)$.
\end{algorithmic}
\end{algorithm}

\section{Additional Synthetic Experiments}
\label{app:sim_additional}

This appendix provides the simulation design and additional diagnostics supporting Section~\ref{sec:Synthetic_Experiments}. Table~\ref{tab:sim_roadmap} summarizes the role of each synthetic experiment. We use \(R1\) and \(R2\) as exact regimes, and \(R3\) and \(R4\) as boundary regimes.

\begin{table}[h]
\centering
\small
\caption{Synthetic experiment roadmap.}
\label{tab:sim_roadmap}
\setlength{\tabcolsep}{4.5pt}
\begin{tabular}{llll}
\toprule
Experiment & Regimes & Purpose & Reported evidence \\
\midrule
Link recovery & \(R1,R2\) & recover identity/log links & Figure~\ref{fig:synthetic_main}(a,b), Tab.~\ref{tab:exact_main} \\
Future deployment & \(R1,R2\) & improve future decisions & Figure~\ref{fig:synthetic_main}(c,d), Tab.~\ref{tab:future_main} \\
Orthogonal ablation & \(R1,R2\) & nuisance robustness & Tab.~\ref{tab:orthogonal_misspec_appendix} \\
Approximate drift & \(R3\) & persistence violation & Fig.~\ref{fig:sim_boundary} \\
No stable geometry & \(R4\) & structural failure & Fig.~\ref{fig:sim_boundary} \\
Overlap stress & \(R1,R2\) & support degradation & Fig.~\ref{fig:sim_overlap} \\
Grid sensitivity & \(R1,R2\) & action discretization robustness & Fig.~\ref{fig:sim_grid_anchor} \\
Anchor sensitivity & \(R1,R2\) & reference-action robustness & Fig.~\ref{fig:sim_grid_anchor} \\
End-to-end ablation & \(R1,R2\) & selection versus refitting & Tab.~\ref{tab:sim_e2e_orthogonal} \\
Persistence magnitude & \(R3\) & structural deployment loss & Tab.~\ref{tab:sim_persistence_scale} \\
Link-selection stress & Transport shift & response-scale selection & Tab.~\ref{tab:sim_link_stress} \\
\bottomrule
\end{tabular}
\end{table}

\subsection{Data-generating process}
\label{app:sim_dgp}

We simulate \(n\) units over \(T\) periods. Each unit has latent state \(S_{i,t}=(U_{i,t},Z_{i,t})\), where \(U_{i,t}\) determines the stable action-sensitive component and \(Z_{i,t}\) drives temporal drift and treatment assignment. The observed history is a nonlinear, time-varying transformation \[H_{i,t}=\phi_t(U_{i,t},Z_{i,t},Y_{i,t-1},A_{i,t-1},\varepsilon^H_{i,t}),\] so the deployable representation is not directly observed. Latent states follow first-order dynamics
\[
U_{i,t+1}=0.82U_{i,t}+\xi^U_{i,t+1},
\qquad
Z_{i,t+1}=0.72Z_{i,t}+\Lambda_A A_{i,t}+\xi^Z_{i,t+1}.
\]
Treatment assignment is observational and time-varying. Unless stated otherwise,
\[
A_{i,t}\mid H_{i,t}\sim
\mathrm{Beta}\!\left(\kappa\mu_{i,t},\kappa(1-\mu_{i,t})\right),
\]
where \(\mu_{i,t}\) depends on latent coordinates, a compressed lagged outcome, and period index. The global confounding-strength parameter multiplies the assignment signal. In \(R1\), we use a bridge-style assignment rule that makes future action ranking more challenging while preserving the raw-scale stable geometry.

Let \(a_0\in(0,1)\) be the anchor and \(d=a-a_0\). The shared action shape is
\[
\rho^\star(u,a)
=
(0.35+0.25\tanh u_1)d
+
(0.18+0.10\tanh u_2)\sin(2\pi a)
+
(-1.10-0.20\tanh(u_1-u_2))d^2,
\]
with \(\rho^\star(u,a_0)=0\). The four regimes are
\[
\begin{array}{ll}
R1: & m_t(h,a)=B_t^{(1)}(h)+\rho^\star(u^\star(h),a),\\[2pt]
R2: & m_t(h,a)=D_t(h)\exp\{\rho^\star(u^\star(h),a)\},\\[2pt]
R3: & m_t(h,a)=B_t^{\rm add}(h)+\rho^\star(u^\star(h),a)
      +\epsilon_{\rm drift}q_t(u^\star(h),a),\\[2pt]
R4: & m_t(h,a)=3.0+B_t^{\rm add}(h)+\rho^\star(u^\star(h),a)
      +\Delta_t\widetilde q_t(u^\star(h),a).
\end{array}
\]
Thus \(R1\) has an exact additive reduction, \(R2\) has an exact multiplicative reduction, \(R3\) introduces controlled action-dependent drift, and \(R4\) breaks stable geometry through strong time-varying shape drift.
For \(R1,R3,R4\), outcomes are generated as clipped additive Gaussian outcomes. For \(R2\), noise is added on the log scale, so the observation model is log-normal. All oracle actions and regrets are evaluated on a dense grid.

\subsection{Candidate library and evaluation metrics}
\label{app:sim_protocol}

The learner selects over a finite representation library
\(\mathcal U_n=\{u^{(1)},u^{(2)},u^{(3)},u^{(4)}\}\):
\[
\begin{array}{lll}
u^{(1)}(H)=U, & \text{oracle representation},\\
u^{(2)}(H)=U+0.4\eta(H), & \text{deterministically perturbed oracle},\\
u^{(3)}(H)=U_1, & \text{under-specified representation},\\
u^{(4)}(H)=f_{\rm rand}(H), & \text{fixed nonlinear random summary}.
\end{array}
\]
This library tests whether the profiled criterion can recover a deployable
low-dimensional geometry rather than a time-entangled summary. The primary metrics are future-period stable IMSE (SIMSE) and regret. SIMSE measures recovery of the stable target on the future split; regret compares the oracle action with the learned action on the evaluation grid. Factual RMSE is reported only as a diagnostic and is not used for model selection.

\subsection{Exact regimes: recovery and deployment}
\label{app:sim_exact}

Table~\ref{tab:exact_main} summarizes the exact-regime sweep. Link recovery is substantially easier than full representation recovery, which is consistent with the paper's target: deployment requires the induced action geometry, not a unique latent representation. Figure~\ref{fig:sim_exact_heatmap} reports link selection accuracy and regret across sample size and confounding strength. Fixed-correct is supplied with the generating link, but still estimates the response shape and profiles its amplitude. Close decision performance relative to this reference indicates low adaptation cost.

\begin{table}[h]
\centering
\small
\caption{
\textbf{Exact-regime recovery summary.} Fixed-correct denotes raw in \(R1\) and log in \(R2\); Fixed-wrong uses the other link. 
}
\label{tab:exact_main}
\setlength{\tabcolsep}{5pt}
\begin{tabular}{lccccc}
\toprule
Regime & LinkAcc & ReprAcc & Ours SIMSE & Fixed-correct SIMSE & Fixed-wrong SIMSE \\
\midrule
\(R1\)
& \(0.786 \pm 0.178\)
& \(0.352 \pm 0.200\)
& \(0.181 \pm 0.019\)
& \(0.160 \pm 0.006\)
& \(0.259 \pm 0.003\) \\
\(R2\)
& \(0.833 \pm 0.148\)
& \(0.195 \pm 0.162\)
& \(0.011 \pm 0.002\)
& \(0.012 \pm 0.002\)
& \(0.0140 \pm 0.004\) \\
\bottomrule
\end{tabular}
\end{table}

\begin{figure}[h]
\centering
\includegraphics[width=0.92\linewidth]{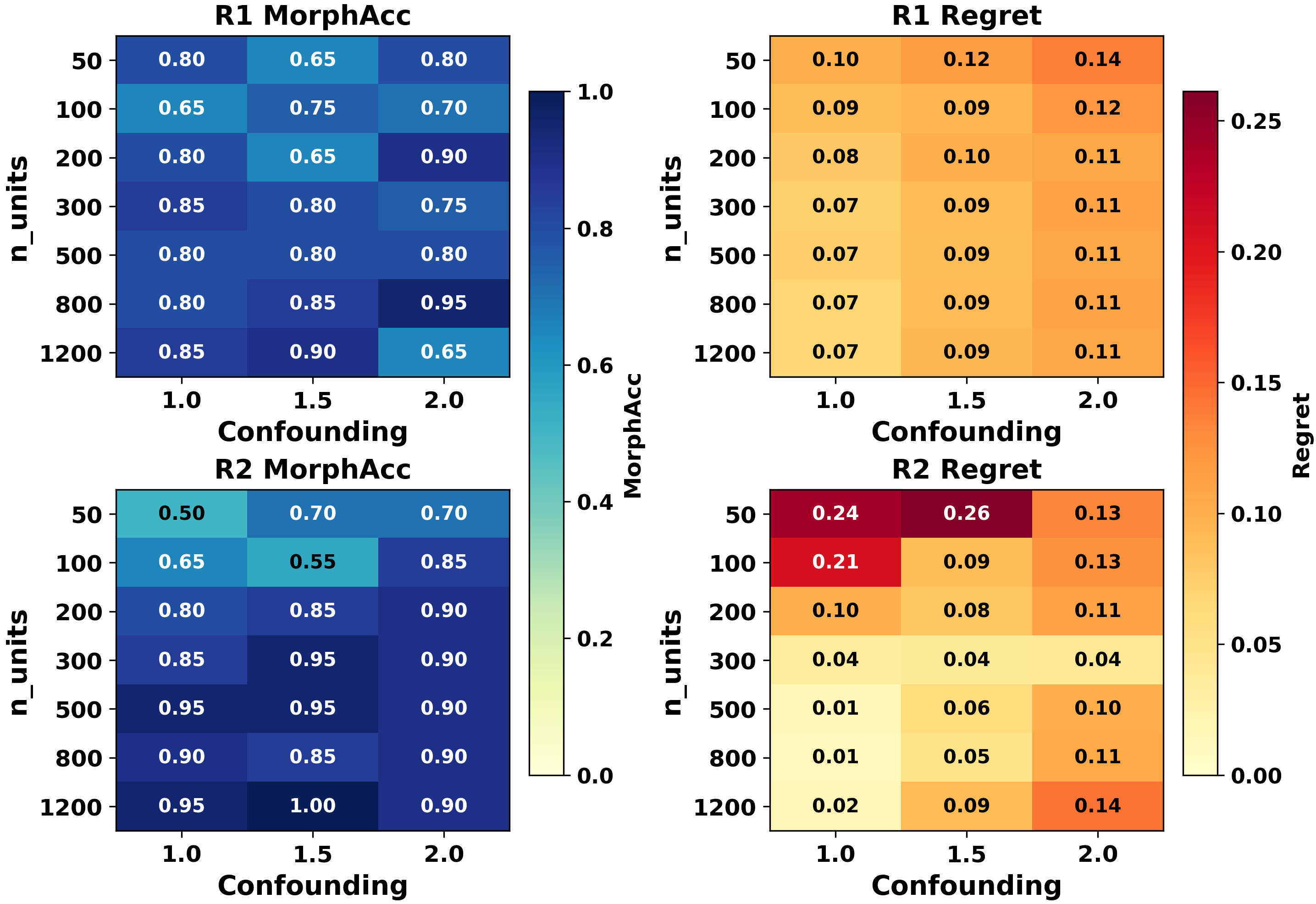}
\caption{
\textbf{Full exact-regime sweep.}
Heatmaps over sample size and confounding strength for \(R1\) and \(R2\). Panels show link selection accuracy and future regret.}
\label{fig:sim_exact_heatmap}
\end{figure}

\subsection{Post-selection and end-to-end orthogonalization}
\label{app:sim_orthogonal}

We freeze the selected link-representation pair and compare the orthogonal refit with direct plug-in under nuisance perturbations. This is a post-selection estimator ablation rather than a second target-selection experiment. The orthogonal estimator consistently yields lower regret and smaller action error than the plug-in estimator. The effect is strongest in the log-linked regime \(R2\): under outcome-side misspecification, the plug-in estimator collapses, whereas the orthogonal estimator remains usable despite increased action error. In \(R1\), SIMSE is similar across estimators, but the orthogonal estimator gives substantially lower regret, showing that predictive curve error and decision quality can diverge.

\begin{table}[h]
\centering
\small
\caption{Post-selection estimator ablation under nuisance misspecification.}
\label{tab:orthogonal_misspec_appendix}
\setlength{\tabcolsep}{5pt}
\begin{tabular}{llccc}
\toprule
Regime / Misspec. & Method & SIMSE \(\downarrow\) & Regret \(\downarrow\) & Action MAE \(\downarrow\) \\
\midrule
\multicolumn{5}{l}{\textbf{R1}} \\
\midrule
none
  & \textsc{Ours}     & \(0.1732 \pm 0.0173\) & \(\mathbf{0.0669 \pm 0.0088}\) & \(\mathbf{0.1205 \pm 0.0072}\) \\
  & \textsc{Plug-in}  & \(\mathbf{0.1697 \pm 0.0163}\) & \(0.2438 \pm 0.0492\) & \(0.2446 \pm 0.0249\) \\
outcome
  & \textsc{Ours}     & \(0.1737 \pm 0.0173\) & \(\mathbf{0.0683 \pm 0.0095}\) & \(\mathbf{0.1231 \pm 0.0075}\) \\
  & \textsc{Plug-in}  & \(\mathbf{0.1689 \pm 0.0174}\) & \(0.1583 \pm 0.0196\) & \(0.1920 \pm 0.0148\) \\
treatment
  & \textsc{Ours}     & \(0.1779 \pm 0.0165\) & \(\mathbf{0.0710 \pm 0.0050}\) & \(\mathbf{0.1250 \pm 0.0046}\) \\
  & \textsc{Plug-in}  & \(\mathbf{0.1697 \pm 0.0163}\) & \(0.2438 \pm 0.0492\) & \(0.2446 \pm 0.0249\) \\
\midrule
\multicolumn{5}{l}{\textbf{R2}} \\
\midrule
none
  & \textsc{Ours}     & \(\mathbf{0.0081 \pm 0.0011}\) & \(\mathbf{0.0154 \pm 0.0091}\) & \(\mathbf{0.0367 \pm 0.0168}\) \\
  & \textsc{Plug-in}  & \(0.0145 \pm 0.0124\) & \(0.3834 \pm 0.0361\) & \(0.4732 \pm 0.0427\) \\
outcome
  & \textsc{Ours}     & \(\mathbf{0.0066 \pm 0.0024}\) & \(\mathbf{0.1326 \pm 0.0560}\) & \(\mathbf{0.2076 \pm 0.0779}\) \\
  & \textsc{Plug-in}  & \(4.8947 \pm 2.3391\) & \(1.2555 \pm 0.0052\) & \(0.4151 \pm 0.0035\) \\
treatment
  & \textsc{Ours}     & \(\mathbf{0.0082 \pm 0.0011}\) & \(\mathbf{0.0163 \pm 0.0101}\) & \(\mathbf{0.0386 \pm 0.0183}\) \\
  & \textsc{Plug-in}  & \(0.0145 \pm 0.0124\) & \(0.3834 \pm 0.0361\) & \(0.4732 \pm 0.0427\) \\
\bottomrule
\end{tabular}
\end{table}

\paragraph{Controlled end-to-end ablation.}
For component attribution, we additionally use a controlled plug-in pilot that omits only the residual-correction term. Localization, smoothing, shape fitting, penalties, and data splits are otherwise held fixed. We independently vary the pilot type at selection and honest refitting: \(O\) denotes orthogonal and \(P\) controlled plug-in, with the first letter indicating selection and the second refitting. Link and representation selection are repeated for each selection-stage pilot type.

\begin{table}[t]
\centering
\small
\setlength{\tabcolsep}{5pt}
\caption{\textbf{End-to-end orthogonalization ablation.}
Mean dense-grid regret under outcome-model misspecification over 30 paired
runs using the full representation library.}
\label{tab:sim_e2e_orthogonal}
\begin{tabular}{@{}lcc@{}}
\toprule
Selection \(\to\) refit & \(R1\) regret \(\downarrow\) & \(R2\) regret \(\downarrow\) \\
\midrule
\(O\to O\) & 0.0754 & 0.1220 \\
\(P\to O\) & 0.0754 & 0.1220 \\
\(O\to P\) & 0.1292 & 0.1444 \\
\(P\to P\) & 0.1292 & 0.1445 \\
\midrule
\(PP-OO\) & 0.0538 [0.0505, 0.0569] & 0.0225 [0.0092, 0.0422] \\
\bottomrule
\end{tabular}
\end{table}

The close regrets of \(OO\) and \(PO\), and of \(OP\) and \(PP\), locate most of the average improvement at orthogonal refitting in these settings. Selection-stage effects are regime-dependent: under moderate \(R2\) outcome-model misspecification, link recovery is 0.57 with orthogonal selection and 0.33 with controlled plug-in selection. Both pilot types select the under-specified representation in the full-library ablation; this comparison does not establish recovery of the generating representation.

\subsection{Boundary regimes: approximate and failed geometry}
\label{app:sim_boundary}

Figure~\ref{fig:sim_boundary} reports the two boundary regimes. In \(R3\), increasing \(\epsilon_{\rm drift}\) adds action-dependent temporal drift to an otherwise stable geometry. Link recovery weakens as the persistence residual grows, which is the expected cost of misspecification. In \(R4\), the action-sensitive shape itself changes over time. This is a structural failure regime; the method should not be expected to uniformly dominate. The purpose of \(R4\) is to show where the stable-geometry target ceases to be appropriate.

\begin{figure}[h]
\centering
\includegraphics[width=0.92\linewidth]{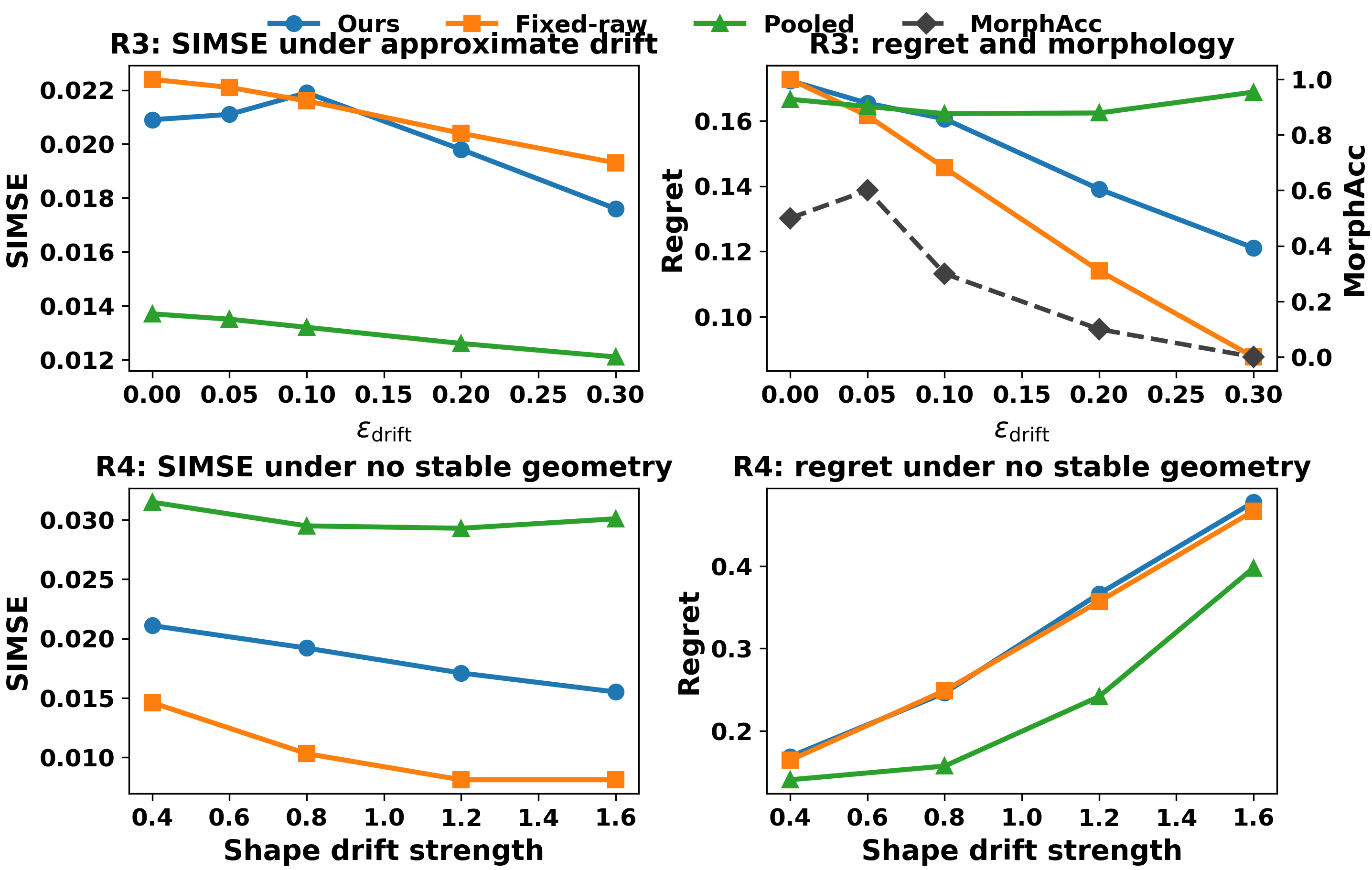}
\caption{
\textbf{Boundary diagnostics.}
Left: approximate persistence violation in \(R3\) as \(\epsilon_{\rm drift}\) increases.
Right: no-stable-geometry stress in \(R4\) as shape drift increases.}
\label{fig:sim_boundary}
\end{figure}

\paragraph{Quantifying persistence violations.}
We additionally distinguish the structural cost of persistence failure from estimation error. For the identity-link \(R3\) sweep, let \(s_t^\dagger(h)\) denote the best positive rescaling of the true shared shape and set \(\xi_t^\dagger(h,a)=\Delta_{0,t}(h,a)-s_t^\dagger(h) \rho^\star(u^\star(h),a)\). The relative violation is
\[
\operatorname{Rel}\xi_\infty
=
\operatorname{median}_{t,h}
\frac{\|\xi_t^\dagger(h,\cdot)\|_\infty}
{\|s_t^\dagger(h)\rho^\star(u^\star(h),\cdot)\|_\infty+\epsilon},
\]
where \(\epsilon>0\) stabilizes the denominator. The oracle-geometry policy maximizes the true shared shape.

\begin{table}[t]
\centering
\small
\setlength{\tabcolsep}{8pt}
\caption{\textbf{Controlled persistence-violation sweep in \(R3\).}
Relative persistence violation and the regret of the oracle-geometry
policy as action-dependent drift increases.}
\label{tab:sim_persistence_scale}
\begin{tabular}{@{}ccc@{}}
\toprule
\(\epsilon_{\mathrm{drift}}\)
& \(\operatorname{Rel}\xi_\infty\)
& Oracle-geometry regret \\
\midrule
0.00 & 0.000 & 0.0000 \\
0.05 & 0.027 & 0.0004 \\
0.10 & 0.055 & 0.0018 \\
0.20 & 0.114 & 0.0095 \\
0.30 & 0.179 & 0.0242 \\
\bottomrule
\end{tabular}
\end{table}

Table~\ref{tab:sim_persistence_scale} shows increasing oracle-geometry regret as the persistence residual grows in this controlled sweep. It measures the cost of transporting the true shared geometry.

\subsection{Overlap stress}
\label{app:sim_overlap}

We stress overlap by increasing assignment severity. This decreases empirical support fraction
and conditional treatment variance. Figure~\ref{fig:sim_overlap} reports support diagnostics
together with SIMSE and regret. Severe support loss harms all estimators; the goal is not to
claim that the method solves positivity violation, but to show predictable degradation and to
contrast it with pooled prediction.

\begin{figure}[h]
\centering
\includegraphics[width=0.8\linewidth]{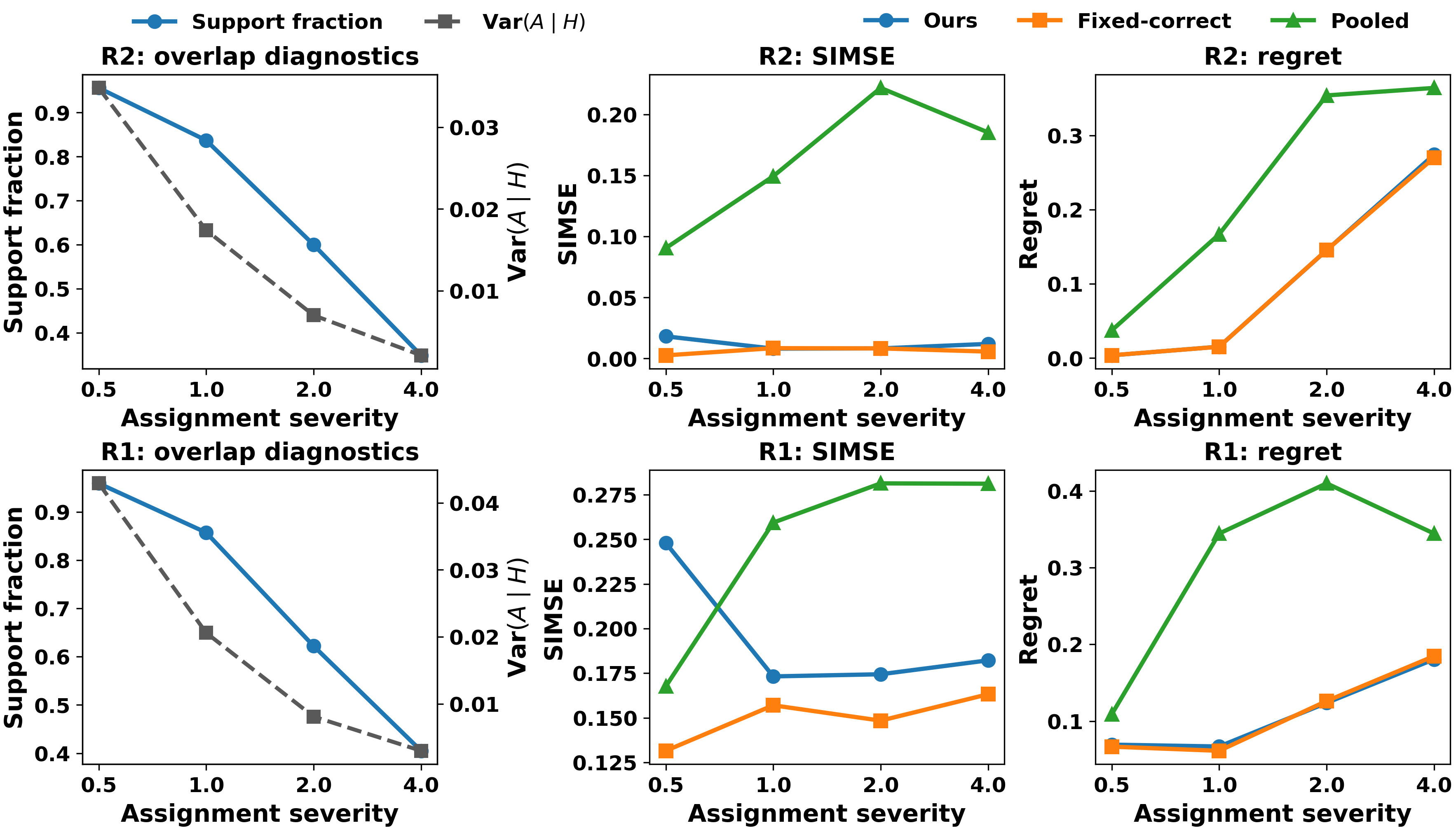}
\caption{Overlap stress.}
\label{fig:sim_overlap}
\end{figure}

\subsection{Grid and anchor sensitivity}
\label{app:sim_grid_anchor}

We further examine whether the conclusions depend on two implementation choices: the action grid used in Stage-II selection and the reference action used for anchoring. For grid sensitivity, the learner selects actions on grids of size \(G\in\{21,51,101,201\}\), while all metrics are evaluated against a fixed 501-point dense grid oracle. This separates optimization discretization from evaluation. For anchor sensitivity, we keep the data-generating process and true response surface fixed, and vary only the reference action used to construct anchored contrasts, \(a_0\in\{0.4,0.5,0.6\}\). Thus raw regret is always computed relative to the same underlying response oracle, while anchored SIMSE may vary with the reference point. Figure~\ref{fig:sim_grid_anchor} reports the results. Grid sensitivity is benign: in both \(R1\) and \(R2\), regret is nearly unchanged once the grid is moderately dense, and the oracle discretization gap is already negligible at the 51-point grid used in the main experiments. Anchor sensitivity is more diagnostic. In \(R1\), decision quality remains stable across the tested anchors. In \(R2\), \(a_0=0.4\) and the default \(a_0=0.5\) both outperform pooled prediction, whereas \(a_0=0.6\) leads to substantially worse decisions even for the fixed-correct log link. This suggests that the anchor should be chosen as a semantically meaningful, high-support reference action rather than treated as an arbitrary tuning parameter.

\begin{figure}[t]
\centering
\includegraphics[width=0.8\linewidth]{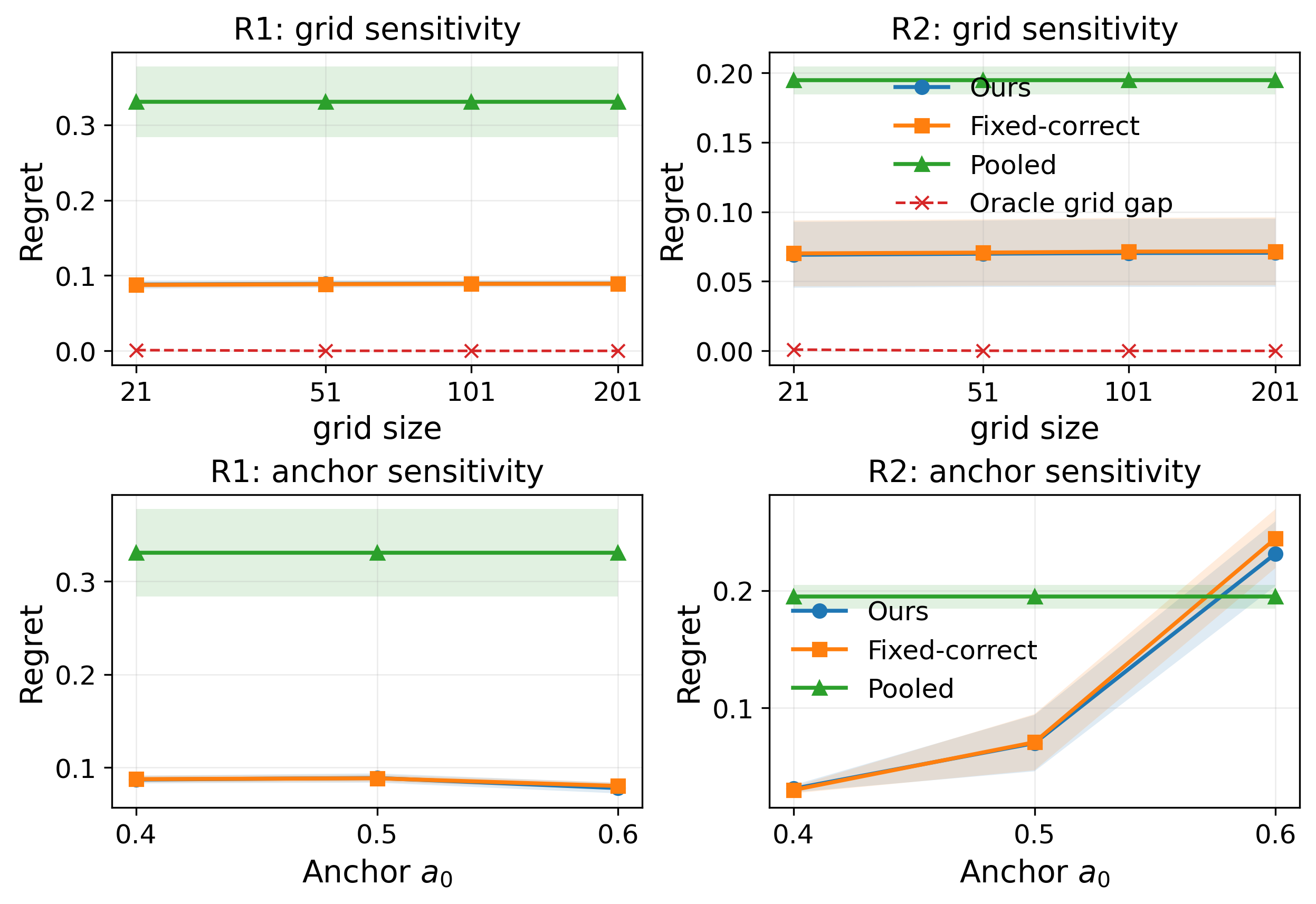}
\caption{
\textbf{Grid and anchor sensitivity.}
Top: sensitivity to the Stage-II action grid. Bottom: sensitivity to the anchor \(a_0\) with the data-generating process and raw-response oracle fixed.}
\label{fig:sim_grid_anchor}
\end{figure}

\subsection{Link selection under transport shift}
\label{app:sim_link_stress}

To examine when selecting the response scale benefits deployment, we consider a supplementary transport-shift setting. The history distribution changes across periods while the generating link admits exact shared geometry, \(\xi_t(h,a)=0\). Under a wrong fixed link, the induced history-dependent curvature lies outside the shared shape class. The given-link reference is supplied with the generating link, not the true response surface or the optimal policy; its amplitude remains profiled as for Ours.

\begin{table}[t]
\centering
\small
\setlength{\tabcolsep}{5pt}
\caption{\textbf{Link selection under transport shift.}
Lower NRegret is better. Paired differences are method minus Ours;
a positive interval excluding zero favors Ours.}
\label{tab:sim_link_stress}
\begin{tabular}{@{}lccc@{}}
\toprule
Method & NRegret & Difference & Paired 95\% CI \\
\midrule
Given-link reference & 0.1384 & \(-0.00008\) & [\(-0.00025\), 0.00008] \\
Ours                 & 0.1385 & ---          & --- \\
Plug-in selector     & 0.1412 & 0.00269      & [0.00160, 0.00378] \\
Fixed identity       & 0.1416 & 0.00309      & [0.00205, 0.00413] \\
Fixed log            & 0.1422 & 0.00368      & [0.00291, 0.00445] \\
\bottomrule
\end{tabular}
\end{table}

Ours has lower mean NRegret than the plug-in selector and both fixed-link alternatives, while the paired interval for the given-link reference includes zero. This complements the exact-regime comparison by examining selection benefits under transport shift, rather than only the cost of adapting to an unknown generating link.

\section{Additional details of the real-world experiment}
\label{app:real_world}

\subsection{Dataset Construction and Evaluation Protocol}
\label{app:dataset}
\noindent\textbf{Data and Action Definition.} The panel spans December 1, 2025 to March 1, 2026, containing $\sim 546.3$M observations across $8.3$M unique items ($\sim 6.0$M daily). The continuous action is the realized effective price normalized by the 14-day median of positive prices, yielding an item-specific anchor $a_0 \approx 1$. The full-sample action distribution is centered near $1$ with a 5th--95th percentile range of $[0.922, 1.095]$.

\noindent\textbf{Support Region and Cohort Filtering.} Analysis is restricted to the empirical common-support region $A^\star = [0.93, 1.08]$ ($[0.90, 1.10]$ for sensitivity), requiring $\geq 10$ positive-price days in the historical window. To control cross-item heterogeneity and ensure sufficient variation, we dynamically retain items satisfying: (i) $\geq 28$ observed days, (ii) $\geq 5$ distinct rounded action bins, (iii) action span $\mathrm{P}90(A) - \mathrm{P}10(A) \geq 0.04$, and (iv) zero-sales rate $\leq 0.10$.

\noindent\textbf{Evaluation Protocol.} We use a rolling-origin chronological setup with six main splits (28-day train, 7-day validation, 7-day test, advanced weekly). An additional split after February 16, 2026, evaluates robustness under a shifted demand regime. After filtering, each split contains $\sim 56$--$58$M training observations ($\sim 2.0$M items) and $12$--$14$M validation/test observations.

\subsection{Metric definition}

We evaluate each method along three complementary axes: factual prediction, future-time deployability, and decision quality.

\paragraph{Factual prediction error.}
For observed item-day samples \((H_i, A_i, Y_i)\) in the evaluation period, we report the RMAE
\begin{equation}
\mathrm{RMAE}= \frac{\sum_{i=1}^n \lvert \hat{y_i} - y_i \rvert }{\sum_{i=1}^{n} y_i},
\end{equation}
where \(\hat m(H,a)\) denotes the learned response surface. RMAE is included as a diagnostic predictive metric, but it is not our primary deployment criterion.

\paragraph{Deployment stability loss.}
To measure whether the learned response geometry remains stable when transported to a future period, we compare the predicted anchored response curve with a future-period anchored target curve on a common action grid \(\mathcal A^\star=\{a_1,\dots,a_K\}\). For a history \(H_i\), define the anchored predicted curve
\begin{equation}
\hat \rho_i(a)
=
\frac{\hat m(H_i,a)}{\hat m(H_i,a_0)},
\qquad a_0=1.
\end{equation}
Let \(\tilde \rho_i^{\,\mathrm{fut}}(a)\) denote the future-period anchored target curve (estimated from held-out future data or its pilot surrogate). We allow a nonnegative scalar recalibration \(\alpha_i \ge 0\) and define the per-sample stability loss as
\begin{equation}
\mathcal L_{\mathrm{stab},i}
=
\min_{\alpha_i \ge 0}
\frac{1}{K}
\sum_{k=1}^K
\Bigl(
\tilde \rho_i^{\,\mathrm{fut}}(a_k)
-
\alpha_i \hat \rho_i(a_k)
\Bigr)^2.
\end{equation}
The reported deployment stability loss is the average over the evaluation set,
\begin{equation}
\mathrm{StabLoss}
=
\frac{1}{n}
\sum_{i=1}^n
\mathcal L_{\mathrm{stab},i}.
\end{equation}
A lower value indicates that the learned shape is more transportable to future environments.

\paragraph{Offline policy value.}
Given a learned surface \(\hat m(H,a)\), the induced pricing policy recommends
\begin{equation}
\hat \pi(H)
=
\arg\max_{a \in \mathcal A^\star}
\hat m(H,a),
\end{equation}
or the corresponding revenue/profit objective if the reward is not sales. We evaluate the policy on future data using a localized doubly robust estimator. Let \(R_i\) denote the observed reward, \(\hat q(H_i,a)\) the nuisance outcome model, \(\hat g(a \mid H_i)\) the estimated generalized propensity score, and \(K_h(\cdot)\) a kernel centered at the recommended action. The value estimator is
\begin{equation}
\widehat V_{\mathrm{DR}}(\hat \pi)
=
\frac{1}{n}
\sum_{i=1}^n
\left[
\hat q\!\left(H_i,\hat \pi(H_i)\right)
+
\frac{
K_h\!\left(A_i-\hat \pi(H_i)\right)
}{
\hat g(A_i \mid H_i)
}
\Bigl(
R_i-\hat q(H_i,A_i)
\Bigr)
\right].
\end{equation}
Higher value indicates a better induced decision rule.

In the real-world study, RMSE measures factual fit, \(\mathrm{StabLoss}\) measures future-time transportability of the learned response geometry, and \(\widehat V_{\mathrm{DR}}\) measures the downstream quality of the induced pricing policy. Our primary emphasis is on the latter two metrics rather than predictive fit alone.

\subsection{Public-data Evaluation on Dominick's}
\label{app:dominicks}

\paragraph{Data and protocol.}
To provide a reproducible complement to the proprietary Taobao study,
we use the 11-brand orange-juice panel distributed with
\texttt{bayesm} 3.1-7 \citep{dominicks}, retaining store--brand--week
records with \texttt{deal=0} and \texttt{feat=0}.
The action is price per ounce divided by its median over the preceding
13 weeks, requiring at least eight observed weeks, with anchor \(a_0=1\).
Observed histories, factual prices, temporal environments, and missingness
are retained; simulated outcomes are not fed back into subsequent histories.
The semi-synthetic evaluation uses four expanding-window folds.
The first uses weeks 53--91 for training, 92--104 for validation, and
105--117 for testing; subsequent folds advance validation and testing
by 13 weeks. The 83 stores are partitioned into disjoint pilot,
selection, and refit sets of 42, 21, and 20 stores, respectively.

The estimator uses brand-based and history-clustered representations
with 11, 22, or 33 cells and a Gaussian action kernel.
A 31-point grid, including \(a_0\), is constructed from the intersection
of training-period 5th--95th percentile action ranges and
\([0.85,1.15]\).
Bandwidth \(b\in\{0.04,0.06,0.08\}\) is selected on the chronological
validation split. We use a representation penalty of \(0.002\) per
log-cell and select the shape-ridge coefficient from
\(\{0,10^{-3},10^{-2}\}\) on the same validation split.
Nuisances are cross-fitted within pilot stores; selection uses validation
data from selection stores, and the selected geometry is refitted using
training and validation data from refit stores.

\paragraph{Semi-synthetic outcome construction.}
We generate only the outcome law, so counterfactual response curves are
known while the observed temporal and treatment-assignment design is
preserved. Let \(u_1(H)\) and \(u_2(H)\) be the standardized lagged
13-week mean sales and one-week-lagged mean competitor price,
respectively. Let \(z_t\) be the standardized weekly market mean
log-sales index from the original panel, where \(t\) denotes calendar
week. Standardization constants and median imputation for missing
competitor prices are fixed using weeks 53--91.
With \(x=x(a)=(a-0.85)/0.30\), \(x_0=0.5\), and \(d=x-x_0\),
the shared anchored shape is
\[
\begin{aligned}
\rho^\star(u,a)
={}&(0.35+0.25\tanh u_1)d \\
&+(0.18+0.10\tanh u_2)
  [\sin(2\pi x)-\sin(2\pi x_0)] \\
&-(1.10+0.20\tanh(u_1-u_2))d^2.
\end{aligned}
\]
Define
\(\eta_t(H)=0.35u_1+0.15u_2+0.60z_t\) and
\(s_t(H)=0.60+1.40\,\sigma(0.80z_t+0.25u_1)\),
where \(\sigma(v)=(1+e^{-v})^{-1}\).
The noiseless response surfaces are
\[
\begin{aligned}
m_t^A(H,a)
&=4.0+1.2\tanh\eta_t(H)
  +s_t(H)\rho^\star(u(H),a), \\
m_t^M(H,a)
&=\exp\!\left\{
  2.5+0.35\tanh\eta_t(H)
  +s_t(H)\rho^\star(u(H),a)
  \right\}, \\
m_t^V(H,a)
&=m_t^M(H,a)
  \exp\!\left\{
  \epsilon_{\mathrm{drift}}q_t(u(H),a)
  \right\},
\end{aligned}
\]
where
\(q_t(u,a)=[\sin(2\pi t/13)+0.5\tanh u_1]d^2\) and
\(\epsilon_{\mathrm{drift}}\in\{0,0.25,0.5,1\}\).
The noiseless surfaces in A and M satisfy exact persistence on the
identity and log scales, respectively; V adds time-varying,
action-dependent curvature on the log scale.

Observed outcomes are generated as
\[
\begin{aligned}
Y^A
&=\max\{m_t^A(H,A)+e_A,10^{-3}\},
& e_A&\sim\mathcal N(0,\sigma_A^2), \\
Y^r
&=m_t^r(H,A)\exp(e_L-0.15^2/2),
& e_L&\sim\mathcal N(0,0.15^2),
\quad r\in\{M,V\},
\end{aligned}
\]
with
\(\sigma_A=0.20\,\mathrm{SD}\{m_t^A(H,A)\}\)
computed once on the earliest training block.
Evaluation uses the noiseless response surfaces.

\paragraph{Counterfactual decision quality.}
For each test row \(i\), let \(m_i(a)\) denote its noiseless response
and \(\widehat a_i\) the recommended action on the fold-specific
evaluation grid \(\mathcal A_{\mathrm{eval}}\). We compute
\[
\operatorname{NRegret}
=
\frac{1}{n}\sum_{i=1}^{n}
\frac{
\max_{a\in\mathcal A_{\mathrm{eval}}}m_i(a)
-m_i(\widehat a_i)
}{
\max_{a\in\mathcal A_{\mathrm{eval}}}m_i(a)
-\min_{a\in\mathcal A_{\mathrm{eval}}}m_i(a)
+10^{-8}
}.
\]
A and M use 10 seeds across four folds; V uses five seeds across
the same folds and four drift levels.
Table~\ref{tab:dominicks_semisynth} reports true future NRegret,
with V averaged across drift levels.
Ours achieves \(0.123\), \(0.063\), and \(0.075\) in A, M, and V,
respectively, and has lower reported mean regret than pooled prediction,
time-aware pooled prediction, VCNet, and Causal Transformer in all
three settings.

\begin{table}[t]
\centering
\small
\caption{\textbf{Semi-synthetic evaluation on Dominick's.}
Each entry reports mean true future NRegret
(reported uncertainty in parentheses); lower is better.
A and M denote exact additive and log-linked persistence.
V averages over four levels of action-dependent shape drift.}
\label{tab:dominicks_semisynth}
\setlength{\tabcolsep}{6pt}
\begin{tabular}{lccc}
\toprule
Method & A & M & V \\
\midrule
Ours
& \(\mathbf{0.123}\;(0.007)\)
& \(\mathbf{0.063}\;(0.003)\)
& \(\mathbf{0.075}\;(0.003)\) \\
Pooled
& \(0.259\;(0.008)\)
& \(0.072\;(0.005)\)
& \(0.089\;(0.004)\) \\
Time-aware pooled
& \(0.248\;(0.008)\)
& \(0.067\;(0.004)\)
& \(0.091\;(0.005)\) \\
VCNet
& \(0.761\;(0.015)\)
& \(0.730\;(0.012)\)
& \(0.697\;(0.009)\) \\
Causal Transformer
& \(0.772\;(0.013)\)
& \(0.731\;(0.012)\)
& \(0.726\;(0.009)\) \\
\bottomrule
\end{tabular}
\end{table}

\paragraph{Link-selection and persistence diagnostics.}
The exact regimes also allow us to assess the cost of selecting the
response scale. In A, Ours has NRegret \(0.1227\), compared with
\(0.1223\) for the fixed-correct identity link and \(0.1258\) for the
fixed-wrong log link. In M, the corresponding values are
\(0.0633\), \(0.0629\), and \(0.0654\), with the log link now
fixed-correct. These comparisons indicate low adaptation cost relative
to a learner supplied with the generating link.

On the original observed panel, we additionally evaluate next-week
factual prediction and rolling persistence diagnostics.
Across 78 weekly origins, a larger held-out profiled residual is
associated with worse next-week geometry loss
(Spearman correlation \(0.37\)).
We use this residual only as a prospective risk-ranking and
falsification diagnostic, rather than as a certificate that persistence
will hold in an unseen future period.



\end{document}